\documentclass[sigconf]{acmart}

\AtBeginDocument{%
  }

\copyrightyear{2026}
\acmYear{2026}
\setcopyright{cc}
\setcctype{by}
\acmConference[KDD 2026] {Proceedings of the 32nd ACM SIGKDD Conference on Knowledge Discovery and Data Mining V.2}{August 9--13, 2026}{Jeju Island, Republic of Korea.}
\acmBooktitle{Proceedings of the 32nd ACM SIGKDD Conference on Knowledge Discovery and Data Mining V.2 (KDD 2026), August 9--13, 2026, Jeju Island, Republic of Korea}
\acmISBN{979-8-4007-2259-2/2026/08}
\acmDOI{10.1145/3770855.3817789}

\usepackage{amsmath,amsfonts,bm}

\def\eqref#1{equation~\ref{#1}}
\def\1{\bm{1}}

\def\vh{{\bm{h}}}

\def\vv{{\bm{v}}}

\def\mD{{\bm{D}}}
\def\mE{{\bm{E}}}
\def\mF{{\bm{F}}}

\def\mH{{\bm{H}}}

\def\mK{{\bm{K}}}
\def\mL{{\bm{L}}}
\def\mM{{\bm{M}}}

\def\mT{{\bm{T}}}

\def\mW{{\bm{W}}}

\DeclareMathAlphabet{\mathsfit}{\encodingdefault}{\sfdefault}{m}{sl}
\SetMathAlphabet{\mathsfit}{bold}{\encodingdefault}{\sfdefault}{bx}{n}

\def\gD{{\mathcal{D}}}
\def\gE{{\mathcal{E}}}

\def\gG{{\mathcal{G}}}

\def\gS{{\mathcal{S}}}

\def\gV{{\mathcal{V}}}

\usepackage{soul}
\usepackage{url}
\usepackage[utf8]{inputenc}
\usepackage{amsmath}
\usepackage{amsthm}
\usepackage{booktabs}
\usepackage{algorithm}
\usepackage{algorithmic}
\usepackage[switch]{lineno}
\usepackage{enumitem}
\usepackage{caption}
\usepackage{graphicx}
\usepackage{natbib}
\usepackage{svg}
\usepackage{multirow}
\usepackage{lipsum}
\usepackage{balance}

\usepackage{subcaption}
\usepackage{cleveref}

\newcommand{\method}{AGDN}

\begin{document}

\title{AGDN: Learning to Solve Traveling Salesman Problem with Anisotropic Graph Diffusion Network}


\author{Bolin Shen}
\authornote{This work was done while the author was at Singapore Management University.}
\affiliation{
  \institution{Florida State University}
  \city{Tallahassee}
  \country{United States}
}
\email{blshen@fsu.edu}

\author{Ziwei Huang}
\affiliation{
  \institution{Singapore Management University}
  \country{Singapore}
}
\email{ziweihuang@smu.edu.sg}

\author{Zhiguang Cao}
\affiliation{
  \institution{Singapore Management University}
  \country{Singapore}
}
\email{zhiguangcao@outlook.com}

\author{Yushun Dong}
\authornote{Corresponding author.}
\affiliation{
  \institution{Florida State University}
  \city{Tallahassee}
  \country{United States}
}
\email{yd24f@fsu.edu}

\renewcommand{\shortauthors}{Bolin Shen, Ziwei Huang, Zhiguang Cao, and Yushun Dong}

\begin{abstract}

% The Traveling Salesman Problem (TSP) remains a fundamental challenge in combinatorial optimization with widespread applications. In recent years, despite the early explorations in solving TSP with graph learning methods, how to better leverage graph structure information for more effective learning remains under-explored. In this paper, we propose the Anisotropic Graph Diffusion Network (AGDN), a novel framework for solving TSP with Graph Neural Networks. Our work addresses two pressing challenges in this domain: (1) the lack of informative topological prior in fully connected TSP graphs and (2) losing connected nodes in the optimal solution after the commonly used graph sparsification. To handle these challenges, we introduce MixScore transition matrix that encodes both node similarity and pairwise distance information, along with an anisotropic graph diffusion mechanism that enables efficient multi-hop information propagation. Through extensive experiments on TSPs of various sizes and node distributions, the proposed framework AGDN demonstrates superior performance compared to existing alternatives while maintaining competitive computational efficiency. Moreover, AGDN shows strong generalization capabilities across different problem sizes and node distributions. Our code is now available at: \url{https://anonymous.4open.science/r/agdn}.

The Traveling Salesman Problem (TSP) is a cornerstone of combinatorial optimization and arises in many practical scenarios. Although graph‐based learning approaches have been explored for TSP, the question of how to exploit graph structure more effectively remains open. We present the Anisotropic Graph Diffusion Network (AGDN), a new Graph Neural Network framework designed to solve TSP. Our method tackles two central difficulties: (1) the lack of informative topological prior in fully connected TSP graphs, and (2) losing connected nodes in the optimal solution after the commonly used graph sparsification techniques. To overcome these issues, we construct a MixScore transition matrix that merges node similarity with pairwise distance, and we develop an anisotropic graph diffusion strategy that supports efficient information exchange across multiple hops. Comprehensive experiments spanning diverse instance sizes and node distributions show that AGDN consistently outperforms existing methods while keeping computation time competitive. Furthermore, AGDN generalizes well to problem sizes and distributions beyond those seen during training. The implementation is publicly available at: \url{https://github.com/LabRAI/AGDN}.

\end{abstract}

\begin{CCSXML}
<ccs2012>
   <concept>
       <concept_id>10003752.10003809.10003635.10010037</concept_id>
       <concept_desc>Theory of computation~Shortest paths</concept_desc>
       <concept_significance>500</concept_significance>
       </concept>
   <concept>
       <concept_id>10010147.10010257.10010293.10010294</concept_id>
       <concept_desc>Computing methodologies~Neural networks</concept_desc>
       <concept_significance>300</concept_significance>
       </concept>
 </ccs2012>
\end{CCSXML}

\ccsdesc[500]{Theory of computation~Shortest paths}
\ccsdesc[300]{Computing methodologies~Neural networks}

\keywords{Traveling Salesman Problem, Graph Neural Networks}

\maketitle

\newcommand\kddavailabilityurl{https://doi.org/10.5281/zenodo.20496880}
\ifdefempty{\kddavailabilityurl}{}{
\begingroup\small\noindent\raggedright\textbf{Resource Availability:}\\
% please change the following context to include multiple artifacts if necessary, including data, models, code, etc.
The source code of this paper has been made publicly available at \url{\kddavailabilityurl}.
\endgroup
}

\section{Introduction}

% \noindent 
% \todo{First, 30+ references in total for THIS SECTION solely}\\

% \noindent 
% \todo{Second, it is illegal to directly use a name (e.g., sparsification, gated mechanism, etc.) without explaining in plain language about what it is first.}\\

% \noindent 
% \todo{Third, for anything in this paper, if you do not put a reason/justification there first, NO more words, NO more sentence, NO more actions. NOTHING should be added then.}\\

% \noindent 
% \todo{You should NOT avoid checking whether the subjective aligns with the verb or not.}\\

% \noindent 
% \todo{prompt: "Please pick out the non-daily words or phrases in the sentences below and list them for me:"} 

% \yd{Potential solutions: (1) xx, where/which/ ..., (2) xxx, i.e., xxx; (3) xxx (e.g., xxx); (4) xxx such as/including xxx; (5) xxx. For example/instance, xxx; (6) Here, xxx (use plain language to simply define something - and use the prompt to check AGAIN after you write it).}

% \yd{NOTICE: solutions should start from something that HAS been discussed.}

% \bs{TSP is like an application, it is difficulty to find application's application}

The Traveling Salesman Problem (TSP) is a well-known combinatorial optimization problem~\citep{korte2011cop} with widespread applications such as circuit design~\citep{kirkpatrick1983circuitdesign} and vehicle routing~\citep{clarke1964vrp}.
% in engineering~\citep{lenstra1975applications-1} (e.g. circuit design~\citep{kirkpatrick1983circuitdesign}, computer board connectivity~\citep{matai2010applications}), as well as important applications in vehicle routing~\citep{clarke1964vrp}, where it holds substantial financial and social benefits~\citep{malandraki1993benefit}.
The objective of a given TSP (referred to as a TSP instance)
% A specific TSP is equivalent to a TSP instance, and the objective of one TSP instance is finding 
is to find the shortest possible route that visits a specified set of cities once and returns to the starting point. Accordingly, the search space of TSP is often modeled as a fully connected graph (i.e., a TSP graph) spanning across different cities~\citep{joshi2019efficient, fu2021generalize}.
% , and the nodes , where each node represents a city. We refer to such a graph and each pair of neighboring nodes in the optimal solution of a TSP as the TSP graph and optimal nodes, respectively. 
% Numerous classic methods have already been proposed to solve the TSP, including exact and heuristic methods~\citep{applegate2009concorde, helsgaun2017lkh} often face challenges such as high computational resource requirements~\citep{ozden2017timeconsume} for large-scale instances and the need for manual effort~\citep{xin2021neurolkh, chen2024finetune}.
Numerous classic methods have already been proposed to handle TSP~\citep{applegate2009concorde, helsgaun2017lkh}. However, these methods typically face challenges such as high computational costs~\citep{ozden2017timeconsume} and heavy manual efforts~\citep{xin2021neurolkh, chen2024finetune}.

% such as exact and heuristic methods~\citep{applegate2009concorde, helsgaun2017lkh} often face challenges such as high computational resource requirements~\citep{ozden2017timeconsume} for large-scale instances and the need for manual effort~\citep{xin2021neurolkh, chen2024finetune}.

% with traditional approaches often relying on the Concorde solver\citep{applegate2009concorde} to obtain exact optimal solutions. Nonetheless, this method demands substantial computational resources, particularly for large-scale TSP instances, leading to an unaffordable time costs. Alternatively, the Lin-Kernighan Heuristic (LKH)\citep{helsgaun2017lkh} is a heuristic-based method that iteratively swap node sequences to search for solutions, known as the k-opt method. While this method offers a near-optimal solution, it requires expert knowledge to fine-tune parameters, such as the number of maximum iteration trials and the number of node swaps.

In recent years, learning-based methods~\citep{kool2018attention, kwon2020pomo, luo2023lehd, fang2024invit, huang2025rethinking, yi2026radar} have emerged as a promising alternative, where Graph Neural Networks (GNNs)~\citep{velivckovic2017gat,kipf2016gcn} have proven particularly effective~\citep{joshi2019efficient, qiu2022dimes, sun2023difusco, zhang2025adversarial, zhang2026hybrid} due to their natural alignment with TSP's topological search space.
% To overcome the limitations of traditional approaches, in recent years, learning-based methods~\citep{kool2018attention, kwon2020pomo, luo2023lehd, fang2024invit} have gained  popularity. Since the search space of a TSP can be naturally considered as a graph, Graph Neural Networks (GNNs)-based models have attracted an increasing research attention and achieved remarkable performance~\citep{joshi2019efficient, qiu2022dimes, sun2023difusco}. 
However, the volume of a TSP's search space grows exponential w.r.t. the total number of nodes. As a consequence, sparsifying the TSP graph is often necessary for the current learning methods to reduce the search space and improve efficiency~\citep{lischka2024less}.
% Handling such a large search space is challenging, making it essential to devise strategies to reduce the search space. Commonly, 
% sparsifying the TSP graph is often necessary ~\citep{lischka2024less}.
% such as KNN are employed for sparsification~\citep{lischka2024less}. 
Nevertheless, such a sparsification process may remove certain edges that exist in the underlying optimal solution. In these cases, while the nodes originally connected by these edges may still be higher-order neighbors, their relationship becomes much more difficult for the commonly used shallow GNNs to capture~\citep{kipf2016gcn, velivckovic2017gat}. Therefore, this sparsification mechanism widely used by the current methods can severely jeopardize the learning performance.
% some important edges, resulting in sparsified graphs with the optimal nodes not being connected. 
% This is because in these cases, while the nodes originally connected by these edges may still be higher-order neighbors, such information becomes much more difficult for the commonly used shallow GNNs to capture~\citep{kipf2016gcn, velivckovic2017gat}. 
% This is because one graph convolution typically aggregates 1-hop information within a single layer~\citep{kipf2016gcn, velivckovic2017gat}. As a consequence, the commonly used shallow GNNs cannot properly extract such information.
% If the information from higher-order neighbors is not encoded in its first-order neighbors, it will be compromised~\citep{abu2019mixhop, feng2023diffuser, wang2020multi}. 
Meanwhile, although stacking multiple convolutional layers can capture multi-hop information, such an approach often leads to other severe issues such as over-smoothing~\citep{rusch2023oversmooth-1, chen2020oversmooth-2,rusch2023oversmoothing-8,cai2020oversmoothing-9} and over-squashing~\citep{alon2020bottleneck, topping2021oversquashing-1,di2023oversquashing-2,giraldo2023oversquashing-3}. Despite the critical limitation of the current sparsification mechanism, only limited studies have discussed such an issue~\citep{xin2021neurolkh} or proposed mechanisms to mitigate it~\citep{lischka2024less}, leaving the problem of losing edges in the underlying optimal solution unsolved.
% Nevertheless, these works still face the key problem of losing key edges that contains in the underlying optimal solution.

% Several studies have proposed mechanisms to mitigate 
% the issue of sparsification. However, these works either xx or xx.
% \cite{xin2021neurolkh} emphasizes that sparsification is crucial for effectively training models, especially on large-scale TSP. However, it merely mentions the importance of sparsification without addressing the aforementioned issue. \cite{lischka2024less} proposes a $K$-Tree-based method~\citep{helsgaun2000ktree} for sparsifying TSP instances, extracting the most significant subgraph. Nonetheless, when dealing with large-scale TSP, it still fails to address the issue of losing high-order neighbor information.

% which can be analyzed using the Laplacian matrix~\citep{merris1994laplacian, chung1997laplacian-1, zhu2021spectral-1} to derive properties such as connectivity, clustering and partitioning, 
Despite the significance of such an issue, it is non-trivial to handle and we mainly face two fundamental challenges: (1) \textit{Lack of an Informative Topological Prior}: Unlike classical graph learning problems that benefit from the topological prior of the input graph data~\citep{velivckovic2017gat, hamilton2017graphsage, xu2018gin}, the TSP graph is often modeled as a fully connected graph without any prior knowledge about the optimal topology~\citep{joshi2019efficient, fu2021generalize, sun2023difusco}. 
% This makes it difficult for the graph learning model to capture key structural information. Even when the adjacency matrix is sparsified using $K$-Nearest Neighbors, the degree of each node is constrained to the same fixed value. 
Such a limitation undermines the capability for the model to capture the solution-relevant information from topology, thereby jeopardizing the quality of learned node representations. (2) \textit{Lack of Information Exchange}: Sparsification may disconnect node pairs that are connected in the optimal solution. If such a situation occurs, the optimal node pairs might still be connected through multi-hop neighbors. However, capturing multi-hop information is challenging~\citep{rusch2023oversmooth-1, chen2020oversmooth-2, alon2020bottleneck}. Therefore, it is crucial to maintain direct connections to efficiently exchange information~\citep{abu2019mixhop, wang2020multi}. An even worse scenario may arise when the optimal node pairs are divided into completely different disconnected graph components. In such a case, graph convolution fails to propagate information between these components~\citep{mcglohon2008disconnect}, jeopardizing information exchange between the optimal node pairs and ultimately affecting the expressiveness of the node embeddings. Therefore, using graph learning to tackle TSP bears pressing challenges and related explorations remain nascent.

% When optimal nodes are multiple hops away, the information from multi-hop neighbors may be diminished~\citep{rusch2023oversmooth-1, chen2020oversmooth-2, alon2020bottleneck}. Therefore, maintaining direct connections with these nodes is crucial to effectively capturing the information of optimal nodes~\citep{abu2019mixhop, wang2020multi}. When such disconnection occurs, an even worse scenario may arise: retaining only $k$ nearest neighbors could result in a disconnected graph. If an optimal node resides in a different connected component, graph convolution will fail to propagate information across these components~\citep{mcglohon2008disconnect}, significantly hampers its expressiveness. Therefore, using graph learning to tackle TSP presents significant challenges and related explorations remain nascent.

To handle the challenges above, in this paper, we propose a novel framework named \textit{\underline{A}nisotropic \underline{G}raph \underline{D}iffusion \underline{N}etwork} (AGDN) for solving TSP. Specifically, we first propose the design of MixScore transition matrix, which simultaneously considers node similarity and pairwise distance information as priors, rather than simply treating the TSP graph as a fully connected graph.
% which leverages node and distance information to compute an attention score that fully characterizes the structure of the entire graph to address the issue of graph convolution becoming ineffective due to disconnected graphs~\citep{mcglohon2008disconnect}. 
Such an approach provides an informative topological prior as the input for our graph learning-based framework and thus effectively handles the challenge of lack of an informative topological prior. Subsequently, we introduced Graph Diffusion with Multi-hop Attention, which contains learnable parameters to control the weight of each hop. This enables immediate access to higher-order neighbor information within a single convolution layer. At its core, we adopted the concept of anisotropy, which separates information propagation for each node into incoming and outgoing directions to improve the characterizing granularity of information flow between nodes~\citep{beaini2021directional}.
Such designs effectively address the critical challenge of lack of information exchange. We summarize our contributions below.

% \vspace{0.05in}
\noindent \textbf{Contributions.} Our contributions can be summarized as threefold: 
\begin{itemize}[leftmargin=2.5em]
    \item \textbf{Problem Characterization.} We pinpoint fundamental limitations in current GNN approaches for TSP, specifically highlighting how graph sparsification disconnects optimal node pairs and how existing methods struggle with multi-hop information propagation in fully connected TSP graphs.
    \item \textbf{Algorithmic Design.} We introduces a novel principled framework AGDN, which combines a novel MixScore transition matrix incorporating structural priors with an anisotropic graph diffusion mechanism that efficiently propagates information bidirectionally through learnable multi-hop attention. 
    % \item \textbf{Empirical Evaluations.} Extensive empirical evaluations demonstrate state-of-the-art performance across different TSP sizes and distributions while showing superior generalization capabilities.
    \item \textbf{Empirical Evaluations.} Extensive empirical evaluations demonstrate strong effectiveness and efficiency in solving TSP, while also exhibiting robust generalization across varying problem sizes, node distributions, and real-world datasets.
\end{itemize}

\section{Preliminary}
\subsection{Notations}
\label{sec: notations}
Throughout this paper, we adopt the following notational conventions: bold uppercase letters (e.g., $\mW$) are used to denote matrices, bold lowercase letters (e.g., $\vh$) are used to denote vectors, and regular lowercase letters (e.g., $d$) are used to represent scalars. 
$\bm{h}_i$ is used to represent the $i$-th element in a given vector $\bm{h}$. We use a vector (e.g., $\bm{h}$) as the subscript of a function parameterized by it, e.g., $f_{\bm{h}}$. 
In this paper, we focus on the two-dimensional plane Euclidean Traveling Salesman Problem, i.e., nodes are positioned in a two-dimensional Euclidean space.
We define the search space for a TSP instance as a fully connected graph  \( \gG = (\gV, \gE) \), where each $\vv\in\gV$ is a two-dimensional vector representing the node coordinates in a two-dimensional space ($\vv\in[0,1]^2$). $e_{ij}\in\gE$ is the Euclidean distance between node $i$ and node $j$ in the two-dimensional space.

\subsection{Traveling Salesman Problem}
The Traveling Salesman Problem is a combinatorial optimization problem in which a salesman is required to visit $n$ cities exactly once and return to the starting city while minimizing the total travel distance. TSP can be formulated as
\[
    \min_{\pi \in \Pi_n} \sum_{i=1}^{n} d_{\pi(i), \pi(i+1)},
\]
where $\Pi_n$ is the set of all possible permutations of the cities, $\pi$ represents a specific tour permutation, and $d_{\pi(i), \pi(i+1)}$ is the distance between two consecutive cities $i$ and $i+1$ in the tour. The objective is to find an optimal tour $\pi^*$ that minimizes the total travel distance.

\subsection{Graph Diffusion Operation}
\label{sec:pre_gdm}

In this paper, we mainly focus on the graph diffusion operations designed with a solid mathematical foundation in Spectral Graph Theory~\citep{chung1997spectral-graph-theory}. In particular, we define graph diffusion using a generalized diffusion matrix:
\[
\mF = \sum_{k=0}^\infty\bm{\theta}_k\mT^k,
\]
where $\mF$ is the diffusion matrix, $\bm{\theta}$ represents the weight controlling the contribution of diffusion steps, and $\mT$ denotes the transition matrix, which represents the graph structure. To ensure the convergence of the diffusion matrix, constraints are often imposed on the weight coefficients, i.e., \(\sum_{k=0}^\infty\boldsymbol{\theta}_k=1, \text{where} \; \boldsymbol{\theta}_k\in[0,1]\). Additionally, $\mT$ should satisfy the condition that its eigenvalue $\lambda_i\in[0,1]$ and it is a stochastic matrix~\citep{xi2018stochastic}. When $\mK$ approaches infinity, the elements of matrix $\mF$ converge. $\mF$ represents the global dependencies of the graph. A graph learning model that leverages $\mF$ for propagating graph features is a spectral-based graph diffusion network, where a well-known example is Graph Diffusion Convolution~\citep{gasteiger2019diffusion}.

\section{Methodology}
% \yd{make sure your figure is on the same page with this paragraph, or at least neighboring pages.}
% \yd{please make sure your tentative Figure 1 follows such a way to draw things. Make this paragraph and figure align with each other as well as you can.}

\setlength{\belowcaptionskip}{-0.1in} % 缩短 caption 与正文之间的空白
\begin{figure*}[t]
    \centering
    \includegraphics[width=0.95\textwidth]{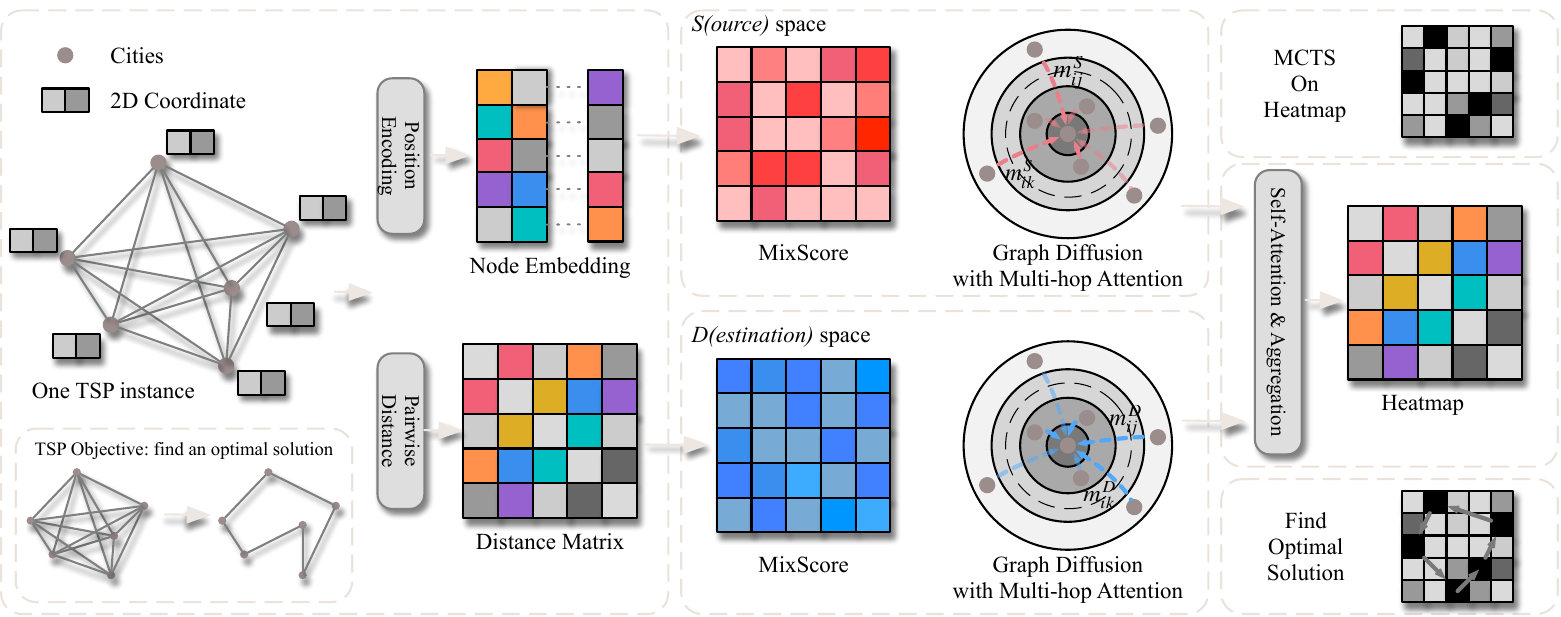}
    \caption{Overview of the proposed AGDN framework. Given a fully connected graph representing a TSP instance, we first construct a distance-based transition matrix that encodes edge-wise transition probabilities. We then compute a novel structural descriptor, MixScore, by jointly considering node features and transition information. MixScore is integrated into the diffusion process to obtain propagated node embeddings. To incorporate directional awareness, we introduce an anisotropic information propagation mechanism that guides embedding updates based on directionality. The resulting directional embeddings are used to generate a heatmap over the graph, which reflects the likelihood of edge participation in the optimal tour. Finally, we employ Monte Carlo Tree Search over the heatmap to identify the optimal TSP solution.}
    \label{fig:agdn}
\end{figure*}

\subsection{Overview}

In this section, we introduce our framework \textit{\underline{A}nisotropic \underline{G}raph \underline{D}iffusion \underline{N}etwork} (AGDN) designed for addressing TSP without compromising to the traditional sparsification process.  We first introduce a modified transition matrix different from the commonly used ones~\citep{kipf2016gcn}, which describes the transition probability of each node based on distances. Next, we propose MixScore, which combines node and transition information to capture the structural information of the entire graph. Then we integrate MixScore with the graph diffusion process to obtain the propagated node embeddings. At the same time, we introduce directional information propagation process, which we refer to as the anisotropic mechanism. Finally, the directional node embeddings are aggregated and utilized for generating heatmaps and identify the optimal TSP solution.

% \subsection{Transition Matrix}
\subsection{MixScore Transition Matrix}
\label{sec:mixscore}
We first propose a transition matrix based on pairwise node distances. Specifically, we notice that if a sparsified adjacency matrix is used for propagating graph information, it inevitably results in disconnected components~\citep{lischka2024less}, preventing information from being transmitted across these components. To avoid such an issue, we utilize a fully connected transition matrix to characterize the graph structure. We first compute the pairwise distance for each element in the set $\gE$, where $e_{ij}$ represents the distance between node $i$ and node $j$. Transition matrix $\mT$ is computed as
% \vspace{-0.01in}
\begin{align*}
    \mT_{ij} = \exp\left(-\frac{e_{ij}}{\sigma}\right),
\end{align*}
% \vspace{-0.02in}
% \[
% \mT_{ij} = \exp\left(-\frac{e_{ij}}{\sigma}\right),
% \]
% \vspace{-0.05in}
where \(\sigma\) is a scaling parameter that controls the sensitivity of the kernel. For smaller values of $\sigma$, transition probabilities between distant points remain significant. In contrast, larger $\sigma$ values cause these probabilities to diminish rapidly, emphasizing local interactions. Compared to traditional transition matrices based on degree, our transition matrix is calculated using pairwise distance. Since we lack a graph structure, distance effectively captures the transition trend, where closer distances indicate higher probabilities. Additionally, to ensure the convergence of the diffusion matrix as the diffusion steps approaches infinity, we propose to normalize the transition matrix with $\mT_\text{rw} = \mD^{-1}\mT$.

Based on the design of our transition matrix, we further propose a novel MixScore transition matrix, which encodes both node similarity information and distance information. Specifically, we utilize the widely adopted information propagation mechanism to ensure that neighboring nodes can access the information from each other. Notably, to achieve higher granularity in the learned node embeddings, we propose to learn and maintain node embeddings separately from both incoming and outgoing directions.
Here, information from the incoming direction is represented in the $\gS$ space, and similarly, information from the outgoing direction is represented in the $\gD$ space. We first map the two-dimensional node coordinates into a higher dimensional latent space using positional encoding~\citep{vaswani2017attention} to obtain node embeddings $\mH \in \mathbb{R}^{n \times d}$.
Subsequently, two multi-layer perceptron (MLP) models, $g^\gS_{\bm{\theta}_s}$ and $g^\gD_{\bm{\theta}_d}$, are used to project $\mH$ into two different spaces to obtain $\mH_\gS$ and $\mH_\gD$. Mathematically, node embeddings in $\mH_\gS$ and $\mH_\gD$ are given as
\begin{align}
    \mH_\gS = g^\gS_{\bm{\theta}_s}(\mH), \text{and} \; \mH_\gD = g^\gD_{\bm{\theta}_d}(\mH),
\end{align}
where $g^\gS_{\bm{\theta}_s}$ and $g^\gD_{\bm{\theta}_d}$ are parameterized by learnable parameter $\bm{\theta}_s$ and $\bm{\theta}_d$, respectively.
% two independent multi-layer perceptions, with their parameters denoted as $\xi_S$ and $\xi_D$, respectively. 
This process helps us obtain anisotropic node embeddings, where the information in the two spaces can be learned separately. 
On the basis of this, we propose to utilize the node embeddings from both directions to define MixScore transition matrix. Here, our intuition is to characterize the propagation of node information from both directions. Therefore, we propose to integrate $\mH_\gS$, $\mH_\gD$, and $\mT_\text{rw}$ together. We captured the similarity between any two points by computing the inner product of $\mH_S$ and $\mH_D$, while $\mT_{rw}$ described the transition probability between them. By integrating these two complementary perspectives, we effectively characterized the overall graph structure. We further propose to use an MLP model $f_{\bm{\omega}}$ to transform the integration to the space of the MixScore transition matrix.
Mathematically, we propose to formulate MixScore transition matrix as

\vspace{-1.5em} % 控制公式前空白
\begin{subequations}
\label{eq:mixscore}
\begin{align}
&\mM_\gS = \text{Softmax}(h_{\bm{\gamma}}((\mH_\gS \mH_\gD^\top) \| \mT_\text{rw}))\ and\\
&\mM_\gD = \text{Softmax}(h_{\bm{\gamma}}((\mH_\gD \mH_\gS^\top) \| \mT_\text{rw}^\top)),
\end{align}
\end{subequations}
\vspace{-1.5em} % 控制公式前空白

where $\|$ is the concatenation operator, and $h_{\bm{\gamma}}$ denotes an MLP model parameterized by $\bm{\gamma}$. With the proposed MixScore transition matrix, we are able to encode information from both directions centered on each node into one matrix to characterize the information exchange on the TSP graph effectively.

\subsection{Anisotropic Graph Diffusion}
\label{sec:agd}

We now propose to use MixScore as the transition matrix in the graph diffusion operation to better extract information from the TSP graph. Specifically, the mathematical formulation of the proposed anisotropic graph diffusion is given as
\begin{align}
    \text{GDMHA}(\bm{\theta}, \mM, \mH) = \sum_{k=0}^{K}\bm{\theta}_k\mM\mH.
\end{align}
Here we propose to adopt $\bm{\theta}_k$ as the multi-hop attention score, which is initialized with a PageRank-based importance score, i.e., $\bm{\theta}_k = \alpha (1-\alpha)^k$, where $\alpha$ is a tunable hyper-parameter.
$K$ denotes the total number of diffusion steps.
Such a design integrated with graph diffusion operation and multi-hop attention mechanism enables us to effectively capture information from multiple hops away for each node in the TSP graph.
Based on such a design, we further propose \textit{Anisotropic Graph Diffusion}, which enables the information to propagate simultaneously in both incoming and outgoing directions for each node. With this design, we are able to characterize the information propagation with higher granularity and thus enhances the model’s ability to capture intricate node relationships. Specifically, the mathematical formulation of the proposed Anisotropic Graph Diffusion is given as
\begin{subequations}
\begin{align}
    &\hat{\mH}_\gS = \text{GDMHA}(\boldsymbol{\theta}_\gS, \mM_\gS, \mH_\gS)\ \text{and} \\
    &\hat{\mH}_\gD = \text{GDMHA}(\boldsymbol{\theta}_\gD, \mM_\gD, \mH_\gD),
\end{align}
\end{subequations}
where the information encoded in the incoming and outgoing directions is separately learned in the $\gS$ and $\gD$ spaces, respectively. We then obtain $\hat{\mH}_{\gS}$ and $\hat{\mH}_{\gD}$ in their respective spaces after information propagation. To summarize, the proposed Anisotropic Graph Diffusion enables the model to capture information from the TSP graph by learning embeddings based on information from separated directions together with the multi-hop attention mechanism.

\subsection{Node Information Aggregation}

To obtain final node embeddings, we propose aggregating the information encoded in the node embeddings from both $\gS$ and $\gD$ spaces. To enable each space to focus on the most relevant features from the other, we propose to adopt the self-attention mechanism~\citep{vaswani2017attention}, using $\hat{\mH}_\gS$ as the query and $\hat{\mH}_{\gD}$ as both the key and the value. Mathematically, the updated embedding $\mH_{\gD}^{\prime}$ is given as

\begin{align}
\label{eq:self-attn}
\mH_\gD^{\prime} = \text{Softmax}\left(\frac{\hat{\mH}_\gS \hat{\mH}_\gD^\top}{\sqrt{d_k}}\right)\hat{\mH}_\gD.
\end{align}

Finally, we aggregate information from the two independent spaces $\hat{\mH}_\gS$ and $\hat{\mH}_\gD$ by using 

\begin{align}
    \tilde{\bm{H}} = f_{\bm{\phi}}(\hat{\mH}_\gS\|\mH_\gD^\prime),
\end{align}
where $f_{\bm{\phi}}$ is an MLP parameterized by $\bm{\phi}$. Such an asymmetrical node information aggregation between the two spaces empirically helps to reduce the instability and effectively integrates the information from the two distinct spaces.

\subsection{Solution Identification and Optimization}
% \todo{math eq}

We finally introduce how to utilize the aggregated node embeddings $\tilde{\bm{H}}$ to obtain the predicted optimal solution to the given TSP problem and conduct the learning optimization. Here, our intuition is to first model edge embeddings between any pair of nodes in the TSP graph based on the learned $\tilde{\bm{H}}$. Then we utilize an attention mechanism to enhance focus on certain specific edges. Finally, we utilize an MLP model to transform the edge embeddings into a predicted heatmap, which highlights the potential routes that involve the optimal route with high probability. Based on such a heatmap, the widely used traditional solvers such as Monte Carlo Tree Search (MCTS)~\citep{browne2012mcts} can be guided to search for an optimal route efficiently.

Our first step here is to obtain the edge embeddings with the learned $\tilde{\bm{H}}$. Specifically, we propose to construct a set of matrices $\mathcal{E}_\text{emb} = \{\bm{E}^{(1)}, \bm{E}^{(1)}, ..., \bm{E}^{(d)}\}$, where the $i$-th matrix encodes the values at the $i$-th dimension of all $n^2$ edge embeddings. We propose to model each $\bm{E}^{(i)}$ as
\begin{align}
    \bm{E}^{(i)} = \tilde{\bm{H}}_{:,i} \tilde{\bm{H}}_{:,i}^{\top} \; (1 \leq i \leq d),
\end{align}
where $\tilde{\bm{H}}_{:,i}$ is used to denote the $i$-th column of the obtained node embedding matrix $\tilde{\bm{H}}$. Our intuition here is to encode pairwise node similarity information to capture their complex relationships in the given TSP graph.

We then propose to utilize an edge-level attention mechanism to re-weight each matrix in the constructed set $\mathcal{E}_\text{emb}$, so that the edges between certain pair of nodes can be highlighted out of the $n^2$ entries to provide better guidance for the proposed model. Specifically, we propose to formulate the attention matrix as 
\begin{align}
\mE_{\text{attn}} = \text{Softmax}\left(\frac{\tilde{\bm{H}}\tilde{\bm{H}}^\top}{\sqrt{d_k}}\right),
\end{align}
where $\mE_{\text{attn}}\in\mathbb{R}^{N\times N}$ is the matrix of the learned edge-level attention score, which helps quantify the significance of edges based on the latent relationships between nodes. Next, utilize $\mE_{\text{attn}}$ to obtain the set of new edge embeddings highlighted by the learned edge-level attention scores. Mathematically, the new edge embedding set is $\tilde{\mathcal{E}}_\text{emb} = \{\tilde{\bm{E}}^{(1)}, \tilde{\bm{E}}^{(1)}, ..., \tilde{\bm{E}}^{(d)}\}$, where $\tilde{\bm{E}}^{(i)} = \bm{E}^{(i)} \odot \mE_{\text{attn}} (1 \leq i \leq d)$ and $\odot$ is the Hadamard product operator.

Finally, we construct an $n \times n \times d$ three-dimensional tensor by stacking all $d$ matrices in $\tilde{\mathcal{E}}_\text{emb}$ together. In this three-dimensional tensor, the vector along the third dimension at position $(i,j)$ on the slice denotes the $d$-dimensional embeddings of the edges connecting the $i$-th and $j$-th node. We propose to utilize an MLP model to perform a binary classification on the existence of each edge in the underlying optimal route, and we construct an $n \times n$ heatmap $\hat{\bm{A}}$ with the predicted positive probability at each edge.

% We now introduce our optimization objective.
Note that we focus on both supervised learning and unsupervised learning settings for model optimization. In supervised learning settings, we directly take the routes yielded by the common solvers~\citep{applegate2009concorde} as the optimal solution, and then we build a heatmap $\bm{A}$ according to the edges presented in the identified optimal solution. To optimize the model, we utilize the widely adopted cross entropy loss to minimize the difference between the output heatmap $\hat{\bm{A}}$ and the heatmap $\bm{A}$ built based on the optimal solution.
In unsupervised learning settings, we follow the common settings in this domain~\citep{min2024utsp} to minimize the summation of the output $\hat{\bm{A}}$ with the $(i, j)$ entry being re-weighted by the corresponding distance value $e_{i,j}$. Here, the intuition is to minimize the average distance between all pairs of nodes with the output probabilities of edge existence, and we introduce details in Appendix.
%
% our output is a soft indicator matrix $\mathbb{T}$, which is used to compute the heatmap and the surrogate loss. Our optimization objective consists of three components: first, ensuring that the sum of each row of $\mathbb{T}$ equals 1 and that $\mathbb{T}$ is column-normalized by Softmax, making it doubly stochastic~\citep{sinkhorn1967doubly} when this term is minimized to 0; second, eliminating self-loops in the heatmap; and third, minimizing the path length obtained by multiplying the distance matrix with the heatmap. \yd{add explanation and references}. 
%
In both cases, the output heatmap $\hat{\bm{A}}$ highlights the critical edges likely to appear in the optimal solution, and we follow the common settings to utilize common solvers to obtain the optimal routes. Specifically, we extracted the value at the second position of the model’s output heatmap logits as our heatmap. Subsequently, for each heatmap, we retained the top $k$ largest values in each row as input for MCTS. The parameters for MCTS were fully configured according to the settings in \cite{min2024utsp}.

\section{Theoretical Analysis}
In this section, we explore the proposed anisotropic mechanism and the MixScore transition matrix, which together significantly enhance multi-hop neighbor communication for solving the TSP. Building on prior work~\citep{beaini2021directional}, where the diffusion distance between nodes is defined as
\(d_t(x, y) := \left( \sum_{z \in \mathcal{V}} \left( q_t(x, z) - q_t(y, z) \right)^2 \right)^{1/2},\)
we extend this notion to incorporate our anisotropic diffusion framework through the following definitions.

\begin{definition}[Anisotropic MixScore Diffusion Distance]
Let \( \mM_\gS, \\ \mM_\gD \in \mathbb{R}^{|\mathcal{V}| \times |\mathcal{V}|} \) denote the two direction MixScore transition matrices, respectively. For each node \( x \in \mathcal{V} \), define the $k$-step transition distribution from \( x \) to \( z \) as \( q_k^{\mM_\gS} \), and similarly for node \( y \in \mathcal{V}\), define the $k$-step transition distribution from \( y \) to \( z \) as \( q_k^{\mM_\gD} \). Then the \emph{anisotropic MixScore diffusion distance} between nodes \( x \) and \( y \) after \( k \) steps is defined as:
\[
d^{\mM}_k(x, y) := \left( \sum_{z \in \mathcal{V}} \left( q_k^{\mM_\gS}(x, z) - q_k^{\mM_\gD}(y, z) \right)^2 \right)^{1/2}.
\]
\end{definition}

This proposition characterizes the diffusion distance induced by directional MixScore matrices, denoted as \( \mM_\gS \) and \( \mM_\gD \), capturing information propagation along distinct directions. Further inspired by~\cite{beaini2021directional}, we generalize the notion of gradient steps to a bi-directional setting and introduce the following definition:

\begin{definition}[Anisotropic Gradient Step]
\label{proposition:bi-grad-step}
Let $x, y, z \in \mathcal{V}$ be nodes such that $z$ is a common neighbor of both $x$ and $y$, and let $\phi \in \mathbb{R}^{|\mathcal{V}|}$ be its first non-trivial left eigenvector. Suppose $\mM_\gS$ and $\mM_\gD$ are defined as in Equation~\ref{eq:mixscore}. We say that $x$ takes a forward gradient step toward $z$ under $\mM_\gS$ if $\phi(z) - \phi(x)$ is maximized among the neighbors of $x$, and that $y$ takes a backward gradient step toward $z$ under $\mM_\gD$ if $\phi(z) - \phi(y)$ is maximized among the neighbors of $y$. Then, performing both updates simultaneously defines an \emph{anisotropic gradient step} toward $z$, yielding updated nodes $(x', y')$.
\end{definition}

Then we arrive at the following theorem, which demonstrates that simultaneous gradient descent along both directions leads to a contraction in diffusion distance:

\vspace{-0.2em} % 控制公式前空白
\begin{theorem}[Anisotropic Gradient Step Reduces Diffusion Distance]
\label{theorem:bi-grad-reduce}
Let $x, y, z \in \mathcal{V}$ be nodes such that $\phi(x) < \phi(z)$ and $\phi(y) < \phi(z)$, and let $(x', y')$ be the result of performing a bi-directional gradient step toward $z$ as defined in Proposition~\ref{proposition:bi-grad-step}. Then there exists a constant $C > 0$ such that for all $t \geq C$, the diffusion distance satisfies:
% \vspace{-0.2em} % 控制公式前空白
\[
d_k^{\mM}(x', y') < d_k^{\mM}(x, y),
\]
% \vspace{-0.2em} % 控制公式前空白
where $d_k^{\mM}(\cdot, \cdot)$ denotes the diffusion distance under the MixScore transition matrix $\mM_\gS$ and $\mM_\gD$.
\end{theorem}

This result highlights the effectiveness of our anisotropic graph diffusion mechanism. By integrating direction-specific MixScore matrices, the framework facilitates faster and more robust long-range information aggregation, especially on large-scale TSP instances, thereby validating the effectiveness of our proposed method. We provide a detailed proof in the Appendix~\ref{appendix:proof}.

\section{Experiments}
In this section, we aim to evaluate the performance of AGDN with extensive experiments. Specifically, we focus on the three research questions below. \textbf{RQ1}: How effective is the Anisotropic Graph Diffusion Network (AGDN) in solving the TSP? \textbf{RQ2}: How efficient is the AGDN when applied to solve TSP? \textbf{RQ3}: What is the generalization performance of AGDN across different problem sizes, diverse node distributions and real-world datasets?

\setlength{\tabcolsep}{2pt}
\begin{table*}[t]
\vspace{0.1in}
\begin{sc}
\begin{center}
\caption{Performance comparison under different TSPs with the number of nodes being 200, 500, and 1000, respectively. RL, SL, UL, G, AS, and BS denotes Reinforcement Learning, Supervised Learning, Unsupervised Learning, Greedy Search, Active Search, and Beam Search, respectively. $\downarrow$ indicates the lower, the better. All results from AGDN are in \textbf{Bold}. Note that we were unable to reproduce the model of \citep{sun2023difusco} under the setting SL+MCTS, a situation also mentioned in \cite{xia2024position}. Accordingly, we use N/A to indicate that the code or experimental settings cannot be reproduced, and there are no results available for reference.}\label{tab:exp-tsp}
\vspace{-0.35em}
% \footnotesize
% \selectfont
\centering
\resizebox{0.98\textwidth}{!}{
\begin{tabular}{ll|ccc|ccc|ccc}
\toprule
\multirow{2}*{Method}&\multirow{2}*{Type}%
&\multicolumn{3}{c|}{TSP-200}&\multicolumn{3}{c|}{TSP-500}&\multicolumn{3}{c}{TSP-1000}\\
&&Length $\downarrow$&Gap $\downarrow$&Time $\downarrow$
&Length $\downarrow$&Gap $\downarrow$&Time $\downarrow$
&Length $\downarrow$&Gap $\downarrow$&Time $\downarrow$
\\
% optimal
\midrule
Concorde& Exact
&10.7&0.00\%&3.44$\mathrm{m}$
&16.6&0.00\%&37.7$\mathrm{m}$
&23.1&0.00\%&6.65$\mathrm{h}$
\\
LKH-3& Heuristics
&10.7&-0.14\%&2.01$\mathrm{m}$
&16.6&0.00\%&11.4$\mathrm{m}$
&23.1&0.00\%&38.1$\mathrm{h}$
\\
% reinforcement learning
\midrule
AM&RL+G
&N/A&N/A&N/A
&20.0&21.0\%&1.51$\mathrm{m}$
&31.2&34.8\%&3.18$\mathrm{m}$
\\
POMO+EAS-Emb&RL+AS+G
&N/A&N/A&N/A
&19.2&16.3\%&12.8$\mathrm{h}$
&N/A&N/A&N/A
\\
POMO+EAS-Tab&RL+AS+G
&N/A&N/A&N/A
&24.5&48.2\%&11.6$\mathrm{h}$
&49.6&114\%&63.5$\mathrm{h}$
\\
DIMES &RL+G
&N/A&N/A&N/A
&18.9 &14.4\% &0.97$\mathrm{m}$ 
&26.6 &15.0\% &2.08$\mathrm{m}$ 
\\
DIMES &RL+AS+G
&N/A&N/A&N/A
&17.8&7.61\%&2.10$\mathrm{h}$
&24.9&7.74\%&4.49$\mathrm{h}$
\\
% supervised learning
\midrule
GatedGCN&SL+BS
&16.2&51.0\%&4.63$\mathrm{m}$
&30.4&83.6\%&38.0$\mathrm{m}$
&51.3&121\%&51.7$\mathrm{m}$
\\
GatedGCN&SL+MCTS
&10.7&0.21\%&0.52$\mathrm{s}$
&16.8&1.28\%&1.46$\mathrm{s}$
&24.8&7.26\%&4.96$\mathrm{s}$
\\
Att-GCN &SL+MCTS 
&10.7&0.16\%&20.6$\mathrm{s}$
& 17.0 & 2.54\% & 2.20$\mathrm{m}$ 
& 23.9 & 3.22\% & 4.10$\mathrm{m}$ 
\\
DIFUSCO & SL+MCTS
&N/A&N/A&N/A
& {16.6} & {0.46\%} & {10.1$\mathrm{m}$}
& {23.4} & {1.17\%} & {24.5$\mathrm{m}$}
\\
\method(Ours) & SL+MCTS
&\textbf{10.7}&\textbf{0.06\%}&\textbf{1.77$\mathrm{s}$}
&\textbf{16.7}&\textbf{0.74\%}&\textbf{2.30$\mathrm{s}$}
&\textbf{23.4}&\textbf{1.28\%}&\textbf{2.40$\mathrm{s}$}
\\
% unsupervised learning
\midrule
Position & SoftDist+MCTS
&N/A&N/A&N/A
&16.8&1.44\%&0.00$\mathrm{s}$
&23.6&2.20\%&0.00$\mathrm{s}$
\\
UTSP &UL+MCTS
&10.7&0.09\%&0.78$\mathrm{s}$
&16.7&0.84\%&0.85$\mathrm{s}$ 
&23.4&1.18\%&1.82$\mathrm{s}$ 
\\
\method(Ours) &UL+MCTS
&\textbf{10.7}&\textbf{0.05\%}&\textbf{0.66$\mathrm{s}$}
&\textbf{16.7}&\textbf{0.64\%}&\textbf{0.64$\mathrm{s}$}
&\textbf{23.3}&\textbf{0.96\%}&\textbf{0.97$\mathrm{s}$}
\\
\bottomrule
\end{tabular}
}
\end{center}
\end{sc}
% \vspace{-0.1in}

\vspace{-0.2em}
\end{table*}
\setlength{\tabcolsep}{6pt}

\subsection{Experimental Setup}

In this subsection, we outline the most representative settings, and details about implementations can be found in Appendix.

\noindent \textbf{Learning Paradigms.} We focus on two types of learning paradigms that are commonly studied in TSPs: (1) Supervised Learning: we first utilize a solver to generate optimal solutions and subsequently derive a ground truth heatmap based on the resulting route. The objective is to predict a heatmap that represents the probability of each edge being an optimal edge, which is then used to guide the Monte Carlo Tree Search (MCTS)~\citep{browne2012mcts} in generating the final path prediction; (2) Unsupervised Learning: Unlike the supervised approach, this paradigm does not rely on a solver to pre-compute the optimal solution for each TSP instance. Instead, we directly calculate the distance between all pairs of nodes re-weighted by the probabilities of edge existence and minimize the average distance between all node pairs.

\noindent
\textbf{Baselines}. 
% To evaluate the effectiveness of AGDN, we compared it with the most representative alternatives in this domain. 
%
% Specifically, t
There are four main categories for solving TSPs: operations research (OR)~\citep{hillier2015or}, reinforcement learning, supervised learning, and unsupervised learning. For each of these mainstream, we select the most representative works as our baselines. For OR methods, we adopt the widely used Concorde~\citep{applegate2009concorde} and LKH~\citep{helsgaun2017lkh} to solve the problems and derive optimal solutions. For reinforcement learning approaches, we adopt AM~\citep{kool2018attention}, POMO~\citep{kwon2020pomo}, and DIMES~\citep{qiu2022dimes}, both of which employ the REINFORCE algorithm~\citep{williams1992reinforce} as a backbone. In the Supervised Learning category, we compare against GatedGCN~\citep{joshi2019efficient}, which is the first to propose using graphs for solving TSP; Att-GCN~\citep{fu2021generalize}, which concatenates heatmaps to enable predictions on TSP of arbitrary sizes; and DIFUSCO~\citep{sun2023difusco}, the first diffusion-based model for TSP. In the Unsupervised Learning category, we compared against Position~\citep{xia2024position} and  UTSP~\citep{min2024utsp}.

\noindent\textbf{Datasets}. Our experiments focus on the widely studied two dimensional Euclidean TSP. The dataset used in this study was derived from the work of \cite{fu2021generalize}, including 1 million training instances and 10,000 test instances for TSP-100. And for TSP-200, TSP-500, and TSP-1000, the dataset includes 3,000 training instances and 128 test instances for each problem size. Additionally, these datasets were utilized for the experiments on size generalization. For the distribution generalization experiments, we follow the settings in \cite{fang2024invit}, whose dataset is widely used to evaluate model performance under different node distributions, including Uniform, Cluster, Explosion, and Implosion. We selected test instances from these four distributions for evaluation on TSP-100 and TSP-1000.

\noindent\textbf{Metrics}. We evaluate the model based on two commonly studied dimensions~\citep{kool2018attention}. (1) Effectiveness: \textsc{Length} represents the length of the optimal solution, with smaller values indicating better performance. To provide a more intuitive understanding of the gap between the near-optimal and optimal solutions, we adopt the \textsc{Gap} metric, which quantifies the relative difference between the predicted route length and the optimal solution, which is given as \textsc{Gap} $= (\text{Predicted Length} - \text{Optimal Length})/\text{Optimal Length} \times 100\%$. Such a metric directly reflects the quality of the identified solution in comparison to the optimal solution. (2) Efficiency: \textsc{Time} represents the duration required for the model’s inference process. Additionally, we will discuss the number of learnable parameters in the model. Overall, we consider a model to be more efficient if it requires less inference time and utilizes fewer parameters.

% \yd{all subsections 5.2, 5.3, and 5.4 below are too short currently and this is unacceptable. For each section, please 1 - use a paragraph to clarify which RQ we wnat to answer; on which metrics we are comparing performances; etc.; 2 - use bold subtitles for each paragraph after the first one to analyze results from multiple dimensions IN DETAILS. The length for each subsection should match this paper https://dl.acm.org/doi/pdf/10.1145/3637528.3671744 in the experiments part. So please READ HOW IT DISCUSSED ITS EXPERIMENTS FIRST AND IMITATE.}
% \yd{make sure }

\setlength{\tabcolsep}{2pt}
\begin{table*}[t]
\vspace{0.15in}
\caption{Performance comparison under different TSP sizes including TSP-200, TSP-500, and TSP-1000.}\label{tab:generalization}
\vspace{-0.05in}
\begin{sc}
\begin{center}
\footnotesize
\centering
\resizebox{0.95\textwidth}{!}{
\begin{tabular}{ll|ccc|ccc|ccc}
\toprule
\multirow{2}*{Method}&\multirow{2}*{Type}%
&\multicolumn{3}{c|}{TSP-200}&\multicolumn{3}{c|}{TSP-500}&\multicolumn{3}{c}{TSP-1000}\\
&&Length $\downarrow$&Gap $\downarrow$&Time $\downarrow$
&Length $\downarrow$&Gap $\downarrow$&Time $\downarrow$
&Length $\downarrow$&Gap $\downarrow$&Time $\downarrow$
\\
% optimal
\midrule
Concorde& Exact
&10.7&0.00\%&3.44$\mathrm{m}$
&16.6&0.00\%&37.7$\mathrm{m}$
&23.1&0.00\%&6.65$\mathrm{h}$
\\
LKH-3& Heuristics
&10.7&-0.14\%&2.01$\mathrm{m}$
&16.6&0.00\%&11.4$\mathrm{m}$
&23.1&0.00\%&38.1$\mathrm{m}$
\\
% Our
\midrule
GatedGCN&SL+MCT
&10.8&0.26\%&0.53$\mathrm{s}$
&16.7&1.16\%&1.47$\mathrm{s}$
&25.9&12.1\%&5.00$\mathrm{s}$
\\
\method(Our)&SL+MCT
&\textbf{10.7}&\textbf{0.17\%}&\textbf{1.73$\mathrm{s}$}
&\textbf{16.7}&\textbf{0.87\%}&\textbf{2.33$\mathrm{s}$}
&\textbf{23.4}&\textbf{1.28\%}&\textbf{2.40$\mathrm{s}$}
\\

\bottomrule
\end{tabular}
}
\end{center}
\end{sc}
% \vspace{-0.1in}

% \vspace{-0.2in}
\end{table*}
\setlength{\tabcolsep}{8pt}

\vspace{-0.2em}

\subsection{Effectiveness Evaluation for AGDN}

To answer \textbf{RQ1}, we conduct experiments on TSP problems of three different sizes to evaluate the effectiveness of AGDN. Specifically, we focus on approaches that integrate heatmap generation with Monte Carlo Tree Search (MCTS)~\citep{browne2012mcts}. This combination has demonstrated superior performance in prior studies~\citep{xia2024position}. To evaluate the effectiveness of the model, we analyze its performance in both Supervised Learning and Unsupervised Learning frameworks~\cref{tab:exp-tsp}. (1) In supervised learning, AGDN achieved the shortest paths in most cases, resulting in the lowest gap relative to the optimal solution. The exception is DIFUSCO~\citep{sun2023difusco}, where the validation issue of the results is widely reported~\citep{xia2024position}. Compared to the widely used GatedGCN with MCTS method, our approach achieved improvements of 71\%, 42\%, and 82\% across different problem sizes, respectively. This demonstrates that our method indeed generates higher-quality heatmaps, further proving its effectiveness. (2) In the unsupervised learning framework, AGDN demonstrates consistent performance improvements. Compared to the state-of-the-art model, UTSP~\citep{min2024utsp}, our approach achieves enhancements of 44\%, 24\%, and 26\% on TSP-200, TSP-500, and TSP-1000, respectively. Since both methods adopt the MCTS algorithm for path searching and use exact same search strategy as UTSP, the impact of the search process can be excluded. These results further demonstrate that AGDN is capable of generating higher-quality heatmaps, reaffirming its effectiveness and its applicability in unsupervised settings. From \cref{fig:heatmap}, we observe that the adjacency matrix obtained after the traditional sparsification mechanism (collected from GateGCN) contains only binary elements, while MixScore effectively learn continuous weights across all node pairs, which helps achieve better performance.

% In summary, AGDN proves to be effective in solving the TSP problem, achieving the best performance across all supervision settings.

% \begin{figure}[t]
%     \centering
%     \includegraphics[height=4.cm]{imgs/plot_efficiency.pdf} \\
%     \vspace{-0.15in}
%     \caption{Efficiency comparison between AGDN and GatedGCN under different number of hops for information aggregation. We observe that AGDN demonstrates clear superiority in both total number of parameters and inference time.}
%     \vspace{-0.1in}
%     \label{fig:efficiency}
% \end{figure}

% % AGDN leverages a diffusion matrix for rapid neighbor aggregation, significantly reducing the number of parameters. Furthermore, since GatedGCN requires complex computations at each layer,

\begin{figure}[t]
    \centering
    % 第一个子图
    \vspace{0.4cm} 
    \hspace{-0.8cm}
    \begin{subfigure}[b]{0.25\textwidth} % 调整宽度占比
        \centering
        \includegraphics[height=3.5cm]{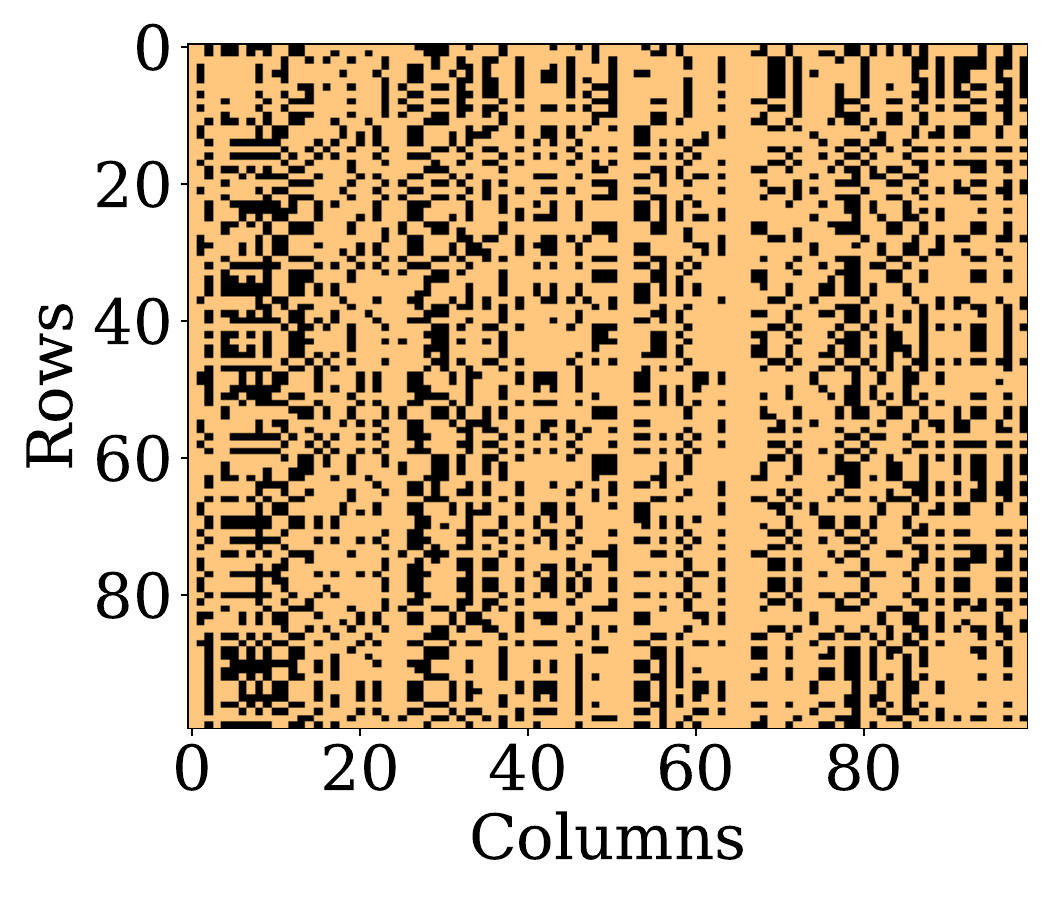} \\
                \vspace{-0.1in}
        \caption{Sparsification Result} % 子图标题
        \label{fig:sub1}
    \end{subfigure}
    % 第二个子图
    \hspace{-0.5cm}
    \begin{subfigure}[b]{0.25\textwidth} % 调整宽度占比
        \centering
        \includegraphics[height=3.5cm]{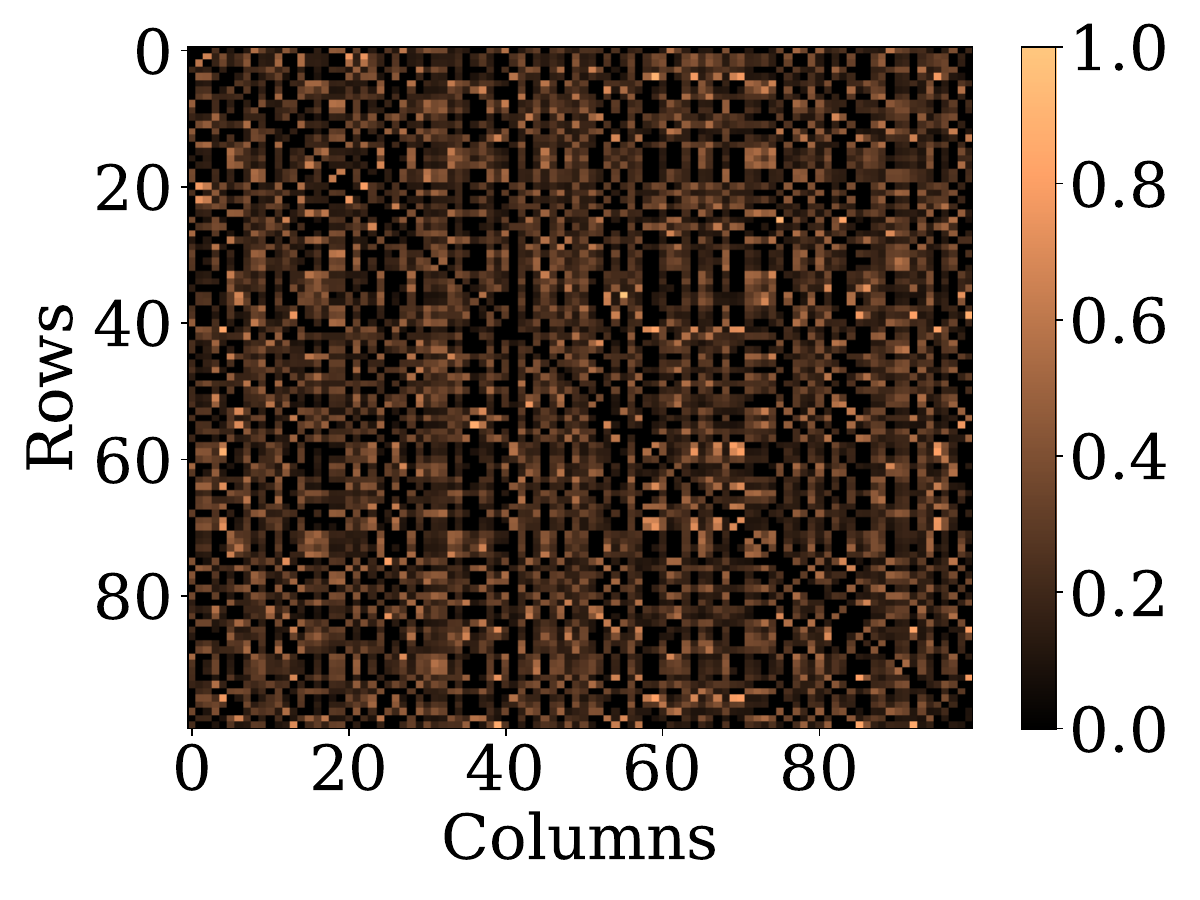} \\
        \vspace{-0.1in}
        \caption{MixScore by AGDN} % 子图标题
        \label{fig:sub2}
    \end{subfigure}
% \vspace{-0.1em}
    \caption{We compare structure information bwteen GatedGCN and AGDN. (a) is derived from GatedGCN, showing the sparsified result obtained through $k$-NN; while (b) is leaned by AGDN. AGDN exhibits better capability to capture the graph structure.}
    \label{fig:heatmap}
% \vspace{-0.4em}
\end{figure}

\begin{figure}[htbp]
\vspace{0.2in}
\setlength{\textfloatsep}{1pt}
    \centering
    \includegraphics[height=4.3cm]{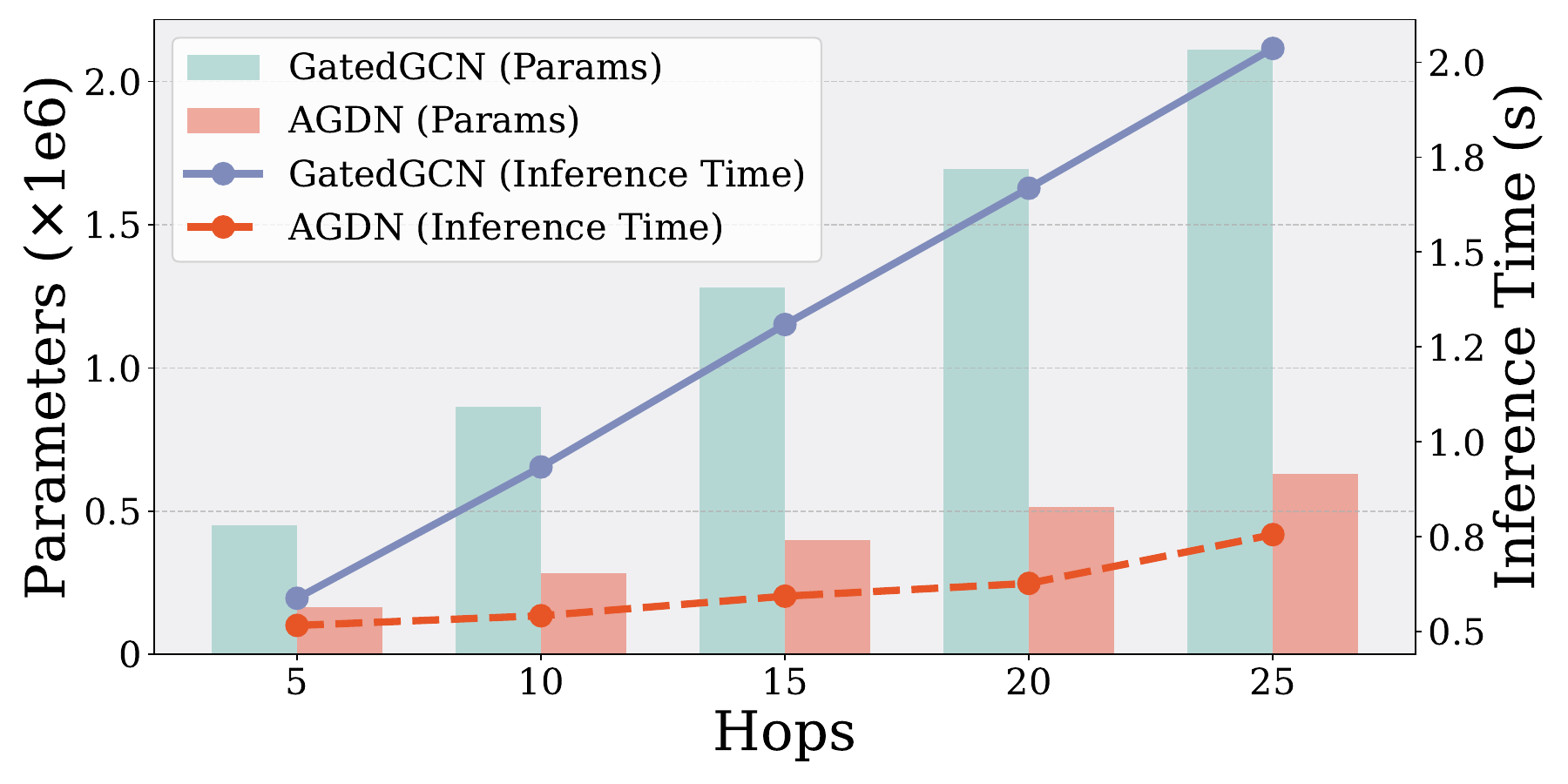} \\
    \vspace{-0.15in}
    \caption{Efficiency comparison between AGDN and GatedGCN under different number of hops for information aggregation. We observe that AGDN demonstrates clear superiority in both total number of parameters and inference time.}
    \label{fig:efficiency}
% \vspace*{-1cm}
\end{figure}

\subsection{Efficiency Evaluation for AGDN}
\label{efficiency_exp}

To answer \textbf{RQ2}, we conducted experiments to compare the efficiency between the proposed AGDN and one of the most representative graph learning-based baselines, GateGCN. We note that that one layer of GatedGCN and one diffusion step in the proposed AGDN enable them to capture information from one-hop farther away for each node in the TSP graph. To ensure fair comparison, we propose to compare efficiency on TSP-200 in cases where 5, 10, 15, 20, and 25 hops are enabled for such information aggregation to enforce the same receptive field, respectively. 

Specifically, we propose to evaluate the efficiency from two perspectives: The number of parameters ($\times 1\text{e}6$) and the inference time (seconds). From the experimental results shown in \cref{fig:efficiency}, we observe that the number of parameters in both models grows w.r.t. the number of hops. However, AGDN requires significantly fewer parameters compared to GatedGCN under the same number of hop (i.e., the same size of receptive field). This is because with the proposed design, AGDN can aggregate $k$-hop information within a single layer without increasing the parameter count.
%
% This is enabled by the use of a shared anisotropic diffusion mechanism that implicitly captures multi-hop dependencies without the need for additional learnable transformations.
%
In contrast, GatedGCN requires one more convolutional layer to aggregate information from one more hop away. Moreover, our inference time outperforms that of GatedGCN in most cases, and such an advantage becomes more significant when the hop number becomes larger. This shows that AGDN is more suitable for solving TSP problems requiring information from farther away for each node. This experiment strongly demonstrates that AGDN achieves a consistently better efficiency compared to the widely used GatedGCN.

% Since AGDN uses far fewer layers, it avoids many computationally intensive operations such as aggregation, normalization, activation, and residual computation, enabling faster aggregation of the same $k$-hop neighborhood information. In addition, we observed that GatedGCN requires each layer to explicitly compute edge embeddings, which is an exponential operation. This results in significant memory usage. During the training phase, AGDN is approximately 6 times faster (at 25 hops) and consumes substantially less memory compared to GatedGCN. Therefore, this experiment strongly demonstrates that AGDN achieves a consistently higher efficiency compared to GatedGCN.

% Therefore, to aggregate $n$-order neighbors, it explicitly requires $n$ convolutional layers, which results in lower efficiency. 

%%%%%%%%%% Different Distribution %%%%%%%%%%
\begin{table}
% \centering
\setlength{\tabcolsep}{8.5pt} % 调整列间距以适应单栏
\caption{Performance comparison under different node distributions.}
\vspace{-0.2cm}
\label{tab:diff-dist}
\begin{sc}
\begin{center}
\small
\centering
\begin{tabular}{l|cc|cc}
\toprule
Distri &\multicolumn{2}{c|}{Uniform} &\multicolumn{2}{c}{Cluster}\\
\midrule
{Method}&{TSP-100}&{TSP-1000}&{TSP-100}&{TSP-1000}\\
% optimal
\midrule
LEHD
&0.57\%&2.76\%&4.51\%&13.7\%
\\
BQ-NCO
&5.90\%&3.91\%&8.86\%&19.2\%
\\
INViT-2V
&1.65\%&6.15\%&3.12\%&9.32\%
\\
INViT-3V
&0.95\%&5.99\%&2.47\%&8.63\%
\\
GCN-MCT
&0.00\%&2.04\%&0.01\%&2.10\%
\\
% Our
\midrule
\textbf{\method}
&\textbf{0.00\%}&\textbf{1.44\%}&\textbf{0.01\%}&\textbf{2.17\%}
\\
\bottomrule

Distri &\multicolumn{2}{c|}{Explosion} &\multicolumn{2}{c}{Implosion}\\
\midrule
{Method}&{TSP-100}&{TSP-1000}&{TSP-100}&{TSP-1000}\\
\midrule
LEHD
&0.68\%&5.99\%&1.17\%&4.25\%
\\
BQ-NCO
&6.41\%&7.21\%&6.40\%&5.43\%
\\
INViT-2V
&1.85\%&9.11\%&1.95\%&6.63\%
\\
INViT-3V
&1.12\%&8.57\%&1.12\%&6.35\%
\\
GCN-MCT
&0.00\%&2.87\%&0.00\%&3.18\%
\\
% Our
\midrule
\textbf{\method}
&\textbf{0.00\%}&\textbf{2.60\%}&\textbf{0.00\%}&\textbf{2.47\%}
\\
\bottomrule

\end{tabular}
\end{center}
\end{sc}

% \vspace{-0.1in}

\end{table}

\begin{table}[htbp]
% \vspace{-0.1in}
\caption{Experimental Results on the Real-World Dataset TSPLIB}
\vspace{-0.05in}
\label{tab:tsplib}
\centering
\small
\begin{tabular}{l@{\hskip 0.02in}|c@{\hskip 0.04in}c@{\hskip 0.02in}|c@{\hskip 0.04in}c@{\hskip 0.02in}|c@{\hskip 0.04in}c}
\toprule
\multirow{2}{*}{\textsc{Method}}
& \multicolumn{2}{c|}{{tsp225}} 
& \multicolumn{2}{c|}{{pcb442}} 
& \multicolumn{2}{c}{{pr1002}} \\
& \textsc{Length} & \textsc{Gap} 
& \textsc{Length} & \textsc{Gap} 
& \textsc{Length} & \textsc{Gap}  \\
\midrule
\textsc{Concorde}    & 8.22 & 0.00\% & 13.4 & 0.00\% & 16.4 & 0.00\% \\
\midrule
\textsc{GatedGCN}    & 8.26 & 0.49\% & 13.5 & 1.15\% & 17.1 & 4.25\% \\
\textbf{\method}        & \textbf{8.25} & \textbf{0.31\%} & \textbf{13.5} & \textbf{0.75\%} & \textbf{16.8} & \textbf{2.76\%} \\
\bottomrule
\end{tabular}
% \vspace*{-1\baselineskip}
\vspace{-0.1in}
\end{table}

\subsection{Evaluation of Generalization Ability}

% To answer \textbf{RQ3}, in this subsection, we evaluate the models' generalization ability on TSPs with different problem sizes (i.e., different total number of nodes) and node distributions, both of which are widely acknowledged as critical~\citep{fang2024invit, luo2023lehd, drakulic2024bq}.
To answer \textbf{RQ3}, in this subsection, we evaluate the models’ generalization ability across three dimensions. First, we examine performance on TSPs with varying problem sizes (i.e., different total numbers of nodes) and node distributions, both of which are widely acknowledged as critical factors for assessing generalization~\citep{fang2024invit, luo2023lehd, drakulic2024bq}. Additionally, we assess the transferability of the models to real-world TSP instances drawn from the TSPLib~\citep{tsplib}, which pose greater structural and distributional complexity than synthetic data and have been widely used in several related works~\cite{goh2024hierarchical,fang2024invit}. This comprehensive evaluation allows us to rigorously test the robustness and scalability of the proposed AGDN framework in both synthetic and practical scenarios.

% This issue has been widely discussed~\citep{fang2024invit, luo2023lehd, drakulic2024bq}, as directly solving large-scale TSP problems is extremely challenging. Using an affordable model to handle tasks of arbitrary sizes and distributions is a common requirement in real-world applications.

% the effectiveness of AGDN has been validated across all TSP sizes. We attribute this to the model’s ability to capture more informative graph structure information. 

\noindent
\textbf{Generalization across Different Problem Sizes.}
We conduct performance comparison by training the learning-based models on TSP-100 and testing them on TSP-200, TSP-500, and TSP-1000, and we present the quantitative results in \cref{tab:generalization}. We observe that the proposed AGDN model consistently shows superior performance in finding high-quality routes (measured by the commonly used \textsc{Gap} metric) compared with the widely used GateGCN model. 
% This reveals AGDN's strengthened ability to capture more informative graph structure information when generalizing across different graph sizes. 
%
Notably, we observed a remarkable improvement of 89\% specifically on TSP-1000, highlighting the effectiveness of AGDN in being generalized to large-scale TSPs. These findings demonstrate that AGDN exhibits superior generalization capabilities. Meanwhile, in terms of efficiency, we observe AGDN shows consistent superiority as discussed in Section~\ref{efficiency_exp}, which further consolidates the advantages of our proposed method.

% Notably, INVIT~\citep{fang2024invit} proposed using rescaling to achieve invariance across different distributions and sizes. 
\noindent
\textbf{Generalization across Different Node Distributions.}
We examine the performance of AGDN on problems with different node distributions in TSP-100 and TSP-1000. Here we focus on comparing the performance on the \textsc{Gap} metric, since it directly reflects the quality of the route identified by the model. We compared the state-of-the-art models across various distributions~\citep{luo2023lehd, drakulic2024bq}. From \cref{tab:diff-dist}, we have observations below. (1) Uniform Distribution: When nodes are evenly distributed across the space, AGDN achieves the best results, validating consistent superiority derived in our previous observations. (2) Cluster Distribution: When nodes are grouped into several dense regions, AGDN remains one of the most competitive models across all baselines. (3) Explosion Distribution: When all nodes are distributed outward toward the periphery, AGDN achieves the best performance in both cases. This can be attributed to the model’s ability to capture and effectively propagate global information. (4) Implosion Distribution: When a subset of nodes is concentrated at a central point while the rest are distributed outward, AGDN continues achieving the best results. The evaluation demonstrates the outstanding ability of the proposed AGDN to consistently exhibit superiority under different node distributions.

\noindent
\textbf{Generalization on Real-World Dataset.}
We further evaluate the generalization ability of our proposed AGDN using real-world datasets from TSPLib~\cite{tsplib}. Specifically, we select three representative TSP instances: tsp225, pcb442, and pr1002, which contain 225, 442, and 1002 nodes, respectively. We compare our method against the baseline model GatedGCN. Both models are trained on the TSP-100 dataset and then evaluated on each of the real-world TSP instances. We report two metrics: \textsc{Length}, which denotes the total length of the predicted shortest path, and \textsc{GAP}, which measures the relative difference between the predicted path length and the known optimal solution. The results are summarized in Table~\ref{tab:tsplib}. As shown in the table, our AGDN method consistently outperforms GatedGCN across all real-world datasets. This experiment provides strong evidence of the superior generalization ability of AGDN, not only does it achieve excellent performance on synthetic data, but it also delivers state-of-the-art results on real-world TSP problems.

\section{Related Works}
\noindent
\textbf{Graph Learning with Diffusion}. Diffusion~\citep{yang2023diffusion,xu2025autostdiff,xu2026synhat} in graph learning~\citep{kipf2016gcn,li2025intellectual,zhao2026graphip,wang2025cega,cheng2025atom,yu2025uqgnn,yu2026health} has been explored through various theoretical frameworks. These approaches can be broadly categorized into two classes based on their mathematical foundation: Spectral-based~\citep{chung1997spectral-graph-theory} and partial/ordinary differential equations(PDE/ODE)-based~\citep{evans2022pde, hartman2002ode} diffusion methods. (1) \textit{Spectral-based Diffusion Methods} leverage the principles of spectral graph theory to enhance the learning process. A prominent example is Graph Diffusion Convolution (GDC)~\citep{gasteiger2019diffusion}, which improves graph learning by introducing a generalized graph convolution. GDC incorporates diffusion methods such as the heat kernel~\citep{davies1989heat-kernel-1, grigoryan2009heat-kernel-2} and personalized PageRank~\citep{langville2004pagerank-1, gasteiger2018pagerank-2} to overcome the limitations of noisy and arbitrarily defined edges in graphs. (2) \textit{PDE/ODE-based Diffusion Methods}: Another line of research leverages mathematical frameworks such as partial differential equations (PDEs)~\citep{evans2022pde} and ordinary differential equations (ODEs)~\citep{hartman2002ode}, to model graph diffusion, treating graph learning as a continuous process~\citep{barton1994continue-process}. Notable examples include Graph Neural Diffusion (GRAND)\citep{chamberlain2021grand}, which utilizes heat equation to describe the temporal evolution of graph signals, and Graph Neural Ordinary Differential Equations (GDE)\citep{poli2019gde}, where the relationship between inputs and outputs is governed by a continuous-time dynamical system, capturing the evolution of graph signals through ordinary differential equations.

\noindent
\textbf{Solving TSPs with Graph Learning}. Solving TSP with graph learning has become a prevalent approach, where three mainstreams of methods are commonly adopted.
(1) \textit{Supervised Learning}: The use of GatedGCN to construct a heatmap to find route from node embeddings was first introduced by \cite{joshi2019efficient}, and numerous subsequent efforts explored graph-based approaches. For example, \cite{fu2021generalize} demonstrated the ability to generalize across TSP instances of arbitrary sizes, removing constraints imposed by the training dataset size. \cite{sun2023difusco} introduced a novel approach using graph diffusion to denoise and generate high-quality solutions. (2) \textit{Reinforcement Learning}: Beyond supervised learning, reinforcement learning frameworks have effectively incorporated graphs. DIMES employed REINFORCE~\citep{williams1992reinforce} to train a GNN, yielding notable performance improvements. (3) \textit{Unsupervised Learning}: \cite{min2024utsp} pioneered the integration of graph models into an unsupervised learning framework, achieving state-of-the-art results without relying on labeled data. These advancements highlight the versatility and potential of graph learning in addressing TSP challenges.
Despite these explorations, they generally overlook the lack of an informative topological prior and fail to avoid the limitation of lacking node information exchange brought by the graph sparsification mechanism. Different from these works, AGDN achieves an informative topological structure without compromising to the limitation of sparsification.

\section{Conclusion}
In this paper, we propose the Anisotropic Graph Diffusion Network (AGDN), a novel framework designed to address the challenges faced by current GNN-based models in solving the Traveling Salesman Problem (TSP). Specifically, we have identified two key issues that current GNNs encounter when addressing TSP, i.e., lacking informative topological prior and information exchange between nodes in the TSP graphs. To address these issues, we propose a novel framework named AGDN integrated with a MixScore transition matrix that encodes both node similarity and pairwise distance information, along with an anisotropic graph diffusion mechanism that enables efficient multi-hop information propagation.
Experiments showed AGDN significantly outperformed existing methods across different problem sizes, diverse node distributions and real-world datasets while maintaining strong generalization capabilities. The results validate AGDN's effectiveness in both supervised and unsupervised learning settings, demonstrating superiority over current state-of-the-art approaches.

% . (1) Lack of informative topological prior: A fully connected TSP graph lacks explicit topological information. A common approach to address this is to sparsify the TSP graph to reduce the search space. (2) Lack of information exchange: Sparsifying the TSP graph may result in disconnected components, which hinders the exchange of information between these components. To address these issues, we propose 

% the use of a MixScore transition matrix, which captures graph structural information by simultaneously considering node similarity and pairwise distance. Additionally, we introduce Anisotropic Graph Diffusion with multi-hop attention to effectively propagate graph features. Extensive experiments validate the effectiveness of our model, and also demonstrate strong generalization capabilities across different graph sizes and distributions. At the same time, we suggest two potential directions for future research. First, AGDN could be applied to other problems, such as the Asymmetric Traveling Salesman Problem (ATSP)~\citep{oncan2009atsp}. Additionally, it could be explored in traditional graph learning tasks (e.g., node classification~\citep{bhagat2011node-classi} and edge prediction~\citep{hasan2011link-pred}).

\section*{Acknowledgments}
This research is supported by the National Research Foundation, Singapore under its AI Singapore Programme (AISG Award No: AISG3-RP-2022-031). We would also like to thank the anonymous reviewers for their constructive feedback.

% balance references
\makeatletter
\let\oldbibitem\bibitem
\newcounter{bibitemcounter}

\def\bibitem{%
  \@ifnextchar[{\bibitem@opt}{\bibitem@noopt}%
}

\def\bibitem@opt[#1]#2{%
  \stepcounter{bibitemcounter}%
  \ifnum\value{bibitemcounter}=59
    \vfill\eject
  \fi
  \oldbibitem[#1]{#2}%
}

\def\bibitem@noopt#1{%
  \stepcounter{bibitemcounter}%
  \ifnum\value{bibitemcounter}=59
    \vfill\eject
  \fi
  \oldbibitem{#1}%
}
\makeatother

\bibliographystyle{ACM-Reference-Format}
\bibliography{references}

@article{kwon2020pomo,
  title={Pomo: Policy optimization with multiple optima for reinforcement learning},
  author={Kwon, Yeong-Dae and Choo, Jinho and Kim, Byoungjip and Yoon, Iljoo and Gwon, Youngjune and Min, Seungjai},
  journal={Advances in Neural Information Processing Systems},
  volume={33},
  pages={21188--21198},
  year={2020}
}

@article{kipf2016gcn,
  title={Semi-supervised classification with graph convolutional networks},
  author={Kipf, Thomas N and Welling, Max},
  journal={arXiv preprint arXiv:1609.02907},
  year={2016}
}

@article{velivckovic2017gat,
  title={Graph attention networks},
  author={Veli{\v{c}}kovi{\'c}, Petar and Cucurull, Guillem and Casanova, Arantxa and Romero, Adriana and Lio, Pietro and Bengio, Yoshua},
  journal={arXiv preprint arXiv:1710.10903},
  year={2017}
}

@article{lischka2024less,
  title={Less Is More-On the Importance of Sparsification for Transformers and Graph Neural Networks for TSP},
  author={Lischka, Attila and Wu, Jiaming and Basso, Rafael and Chehreghani, Morteza Haghir and Kulcs{\'a}r, Bal{\'a}zs},
  journal={arXiv preprint arXiv:2403.17159},
  year={2024}
}

@article{kool2018attention,
  title={Attention, learn to solve routing problems!},
  author={Kool, Wouter and Van Hoof, Herke and Welling, Max},
  journal={arXiv preprint arXiv:1803.08475},
  year={2018}
}

@article{joshi2019efficient,
  title={An efficient graph convolutional network technique for the travelling salesman problem},
  author={Joshi, Chaitanya K and Laurent, Thomas and Bresson, Xavier},
  journal={arXiv preprint arXiv:1906.01227},
  year={2019}
}

@inproceedings{fu2021generalize,
  title={Generalize a small pre-trained model to arbitrarily large tsp instances},
  author={Fu, Zhang-Hua and Qiu, Kai-Bin and Zha, Hongyuan},
  booktitle={Proceedings of the AAAI conference on artificial intelligence},
  volume={35},
  number={8},
  pages={7474--7482},
  year={2021}
}

@article{wang2020multi,
  title={Multi-hop attention graph neural network},
  author={Wang, Guangtao and Ying, Rex and Huang, Jing and Leskovec, Jure},
  journal={arXiv preprint arXiv:2009.14332},
  year={2020}
}

@article{applegate2009concorde,
  title={Certification of an optimal TSP tour through 85,900 cities},
  author={Applegate, David L and Bixby, Robert E and Chv{\'a}tal, Va{\v{s}}ek and Cook, William and Espinoza, Daniel G and Goycoolea, Marcos and Helsgaun, Keld},
  journal={Operations Research Letters},
  volume={37},
  number={1},
  pages={11--15},
  year={2009},
  publisher={Elsevier}
}

@article{helsgaun2017lkh,
  title={An extension of the Lin-Kernighan-Helsgaun TSP solver for constrained traveling salesman and vehicle routing problems},
  author={Helsgaun, Keld},
  journal={Roskilde: Roskilde University},
  volume={12},
  pages={966--980},
  year={2017}
}

@article{qiu2022dimes,
  title={Dimes: A differentiable meta solver for combinatorial optimization problems},
  author={Qiu, Ruizhong and Sun, Zhiqing and Yang, Yiming},
  journal={Advances in Neural Information Processing Systems},
  volume={35},
  pages={25531--25546},
  year={2022}
}

@article{sun2023difusco,
  title={Difusco: Graph-based diffusion solvers for combinatorial optimization},
  author={Sun, Zhiqing and Yang, Yiming},
  journal={Advances in Neural Information Processing Systems},
  volume={36},
  pages={3706--3731},
  year={2023}
}

@article{min2024utsp,
  title={Unsupervised learning for solving the travelling salesman problem},
  author={Min, Yimeng and Bai, Yiwei and Gomes, Carla P},
  journal={Advances in Neural Information Processing Systems},
  volume={36},
  year={2024}
}

@article{alon2020bottleneck,
  title={On the bottleneck of graph neural networks and its practical implications},
  author={Alon, Uri and Yahav, Eran},
  journal={arXiv preprint arXiv:2006.05205},
  year={2020}
}

@article{gasteiger2019diffusion,
  title={Diffusion improves graph learning},
  author={Gasteiger, Johannes and Wei{\ss}enberger, Stefan and G{\"u}nnemann, Stephan},
  journal={Advances in neural information processing systems},
  volume={32},
  year={2019}
}

@inproceedings{beaini2021directional,
  title={Directional graph networks},
  author={Beaini, Dominique and Passaro, Saro and L{\'e}tourneau, Vincent and Hamilton, Will and Corso, Gabriele and Li{\`o}, Pietro},
  booktitle={International Conference on Machine Learning},
  pages={748--758},
  year={2021},
  organization={PMLR}
}

@article{clarke1964vrp,
  title={Scheduling of vehicles from a central depot to a number of delivery points},
  author={Clarke, Geoff and Wright, John W},
  journal={Operations research},
  volume={12},
  number={4},
  pages={568--581},
  year={1964},
  publisher={Informs}
}

@article{kirkpatrick1983circuitdesign,
  title={Optimization by simulated annealing},
  author={Kirkpatrick, Scott and Gelatt Jr, C Daniel and Vecchi, Mario P},
  journal={science},
  volume={220},
  number={4598},
  pages={671--680},
  year={1983},
  publisher={American association for the advancement of science}
}

@inproceedings{abu2019mixhop,
  title={Mixhop: Higher-order graph convolutional architectures via sparsified neighborhood mixing},
  author={Abu-El-Haija, Sami and Perozzi, Bryan and Kapoor, Amol and Alipourfard, Nazanin and Lerman, Kristina and Harutyunyan, Hrayr and Ver Steeg, Greg and Galstyan, Aram},
  booktitle={international conference on machine learning},
  pages={21--29},
  year={2019},
  organization={PMLR}
}

@article{xu2018gin,
  title={How powerful are graph neural networks?},
  author={Xu, Keyulu and Hu, Weihua and Leskovec, Jure and Jegelka, Stefanie},
  journal={arXiv preprint arXiv:1810.00826},
  year={2018}
}

@article{xin2021neurolkh,
  title={Neurolkh: Combining deep learning model with lin-kernighan-helsgaun heuristic for solving the traveling salesman problem},
  author={Xin, Liang and Song, Wen and Cao, Zhiguang and Zhang, Jie},
  journal={Advances in Neural Information Processing Systems},
  volume={34},
  pages={7472--7483},
  year={2021}
}

@article{ozden2017timeconsume,
  title={Solving large batches of traveling salesman problems with parallel and distributed computing},
  author={Ozden, SG and Smith, Alice E and Gue, Kevin R},
  journal={Computers \& Operations Research},
  volume={85},
  pages={87--96},
  year={2017},
  publisher={Elsevier}
}

@article{chen2024finetune,
  title={Efficient meta neural heuristic for multi-objective combinatorial optimization},
  author={Chen, Jinbiao and Wang, Jiahai and Zhang, Zizhen and Cao, Zhiguang and Ye, Te and Chen, Siyuan},
  journal={Advances in Neural Information Processing Systems},
  volume={36},
  year={2024}
}

@article{luo2023lehd,
  title={Neural combinatorial optimization with heavy decoder: Toward large scale generalization},
  author={Luo, Fu and Lin, Xi and Liu, Fei and Zhang, Qingfu and Wang, Zhenkun},
  journal={Advances in Neural Information Processing Systems},
  volume={36},
  pages={8845--8864},
  year={2023}
}

@article{fang2024invit,
  title={INViT: A Generalizable Routing Problem Solver with Invariant Nested View Transformer},
  author={Fang, Han and Song, Zhihao and Weng, Paul and Ban, Yutong},
  journal={arXiv preprint arXiv:2402.02317},
  year={2024}
}

@article{rusch2023oversmooth-1,
  title={A survey on oversmoothing in graph neural networks},
  author={Rusch, T Konstantin and Bronstein, Michael M and Mishra, Siddhartha},
  journal={arXiv preprint arXiv:2303.10993},
  year={2023}
}

@inproceedings{chen2020oversmooth-2,
  title={Measuring and relieving the over-smoothing problem for graph neural networks from the topological view},
  author={Chen, Deli and Lin, Yankai and Li, Wei and Li, Peng and Zhou, Jie and Sun, Xu},
  booktitle={Proceedings of the AAAI conference on artificial intelligence},
  volume={34},
  number={04},
  pages={3438--3445},
  year={2020}
}

@article{xi2018stochastic,
  title={Linear convergence in optimization over directed graphs with row-stochastic matrices},
  author={Xi, Chenguang and Mai, Van Sy and Xin, Ran and Abed, Eyad H and Khan, Usman A},
  journal={IEEE Transactions on Automatic Control},
  volume={63},
  number={10},
  pages={3558--3565},
  year={2018},
  publisher={IEEE}
}

@inproceedings{mcglohon2008disconnect,
  title={Weighted graphs and disconnected components: patterns and a generator},
  author={McGlohon, Mary and Akoglu, Leman and Faloutsos, Christos},
  booktitle={Proceedings of the 14th ACM SIGKDD international conference on Knowledge discovery and data mining},
  pages={524--532},
  year={2008}
}

@article{vaswani2017attention,
  title={Attention is all you need},
  author={Vaswani, A},
  journal={Advances in Neural Information Processing Systems},
  year={2017}
}

@book{korte2011cop,
  title={Combinatorial optimization},
  author={Korte, Bernhard H and Vygen, Jens and Korte, B and Vygen, J},
  volume={1},
  year={2011},
  publisher={Springer}
}

@article{williams1992reinforce,
  title={Simple statistical gradient-following algorithms for connectionist reinforcement learning},
  author={Williams, Ronald J},
  journal={Machine learning},
  volume={8},
  pages={229--256},
  year={1992},
  publisher={Springer}
}

@article{xia2024position,
  title={Position: Rethinking Post-Hoc Search-Based Neural Approaches for Solving Large-Scale Traveling Salesman Problems},
  author={Xia, Yifan and Yang, Xianliang and Liu, Zichuan and Liu, Zhihao and Song, Lei and Bian, Jiang},
  journal={arXiv preprint arXiv:2406.03503},
  year={2024}
}

@book{hillier2015or,
  title={Introduction to operations research},
  author={Hillier, Frederick S and Lieberman, Gerald J},
  year={2015},
  publisher={McGraw-Hill}
}

@article{drakulic2024bq,
  title={Bq-nco: Bisimulation quotienting for efficient neural combinatorial optimization},
  author={Drakulic, Darko and Michel, Sofia and Mai, Florian and Sors, Arnaud and Andreoli, Jean-Marc},
  journal={Advances in Neural Information Processing Systems},
  volume={36},
  year={2024}
}

@book{evans2022pde,
  title={Partial differential equations},
  author={Evans, Lawrence C},
  volume={19},
  year={2022},
  publisher={American Mathematical Society}
}

@book{hartman2002ode,
  title={Ordinary differential equations},
  author={Hartman, Philip},
  year={2002},
  publisher={SIAM}
}

@book{chung1997spectral-graph-theory,
  title={Spectral graph theory},
  author={Chung, Fan RK},
  volume={92},
  year={1997},
  publisher={American Mathematical Soc.}
}

@inproceedings{chamberlain2021grand,
  title={Grand: Graph neural diffusion},
  author={Chamberlain, Ben and Rowbottom, James and Gorinova, Maria I and Bronstein, Michael and Webb, Stefan and Rossi, Emanuele},
  booktitle={International conference on machine learning},
  pages={1407--1418},
  year={2021},
  organization={PMLR}
}

@article{poli2019gde,
  title={Graph neural ordinary differential equations},
  author={Poli, Michael and Massaroli, Stefano and Park, Junyoung and Yamashita, Atsushi and Asama, Hajime and Park, Jinkyoo},
  journal={arXiv preprint arXiv:1911.07532},
  year={2019}
}

@book{davies1989heat-kernel-1,
  title={Heat kernels and spectral theory},
  author={Davies, Edward Brian},
  number={92},
  year={1989},
  publisher={Cambridge university press}
}

@book{grigoryan2009heat-kernel-2,
  title={Heat kernel and analysis on manifolds},
  author={Grigoryan, Alexander},
  volume={47},
  year={2009},
  publisher={American Mathematical Soc.}
}

@article{langville2004pagerank-1,
  title={Deeper inside pagerank},
  author={Langville, Amy N and Meyer, Carl D},
  journal={Internet Mathematics},
  volume={1},
  number={3},
  pages={335--380},
  year={2004},
  publisher={Taylor \& Francis}
}

@article{gasteiger2018pagerank-2,
  title={Predict then propagate: Graph neural networks meet personalized pagerank},
  author={Gasteiger, Johannes and Bojchevski, Aleksandar and G{\"u}nnemann, Stephan},
  journal={arXiv preprint arXiv:1810.05997},
  year={2018}
}

@article{barton1994continue-process,
  title={Modeling of combined discrete/continuous processes},
  author={Barton, Paul I and Pantelides, Constantinos C},
  journal={AIChE journal},
  volume={40},
  number={6},
  pages={966--979},
  year={1994},
  publisher={Wiley Online Library}
}

@article{browne2012mcts,
  title={A survey of monte carlo tree search methods},
  author={Browne, Cameron B and Powley, Edward and Whitehouse, Daniel and Lucas, Simon M and Cowling, Peter I and Rohlfshagen, Philipp and Tavener, Stephen and Perez, Diego and Samothrakis, Spyridon and Colton, Simon},
  journal={IEEE Transactions on Computational Intelligence and AI in games},
  volume={4},
  number={1},
  pages={1--43},
  year={2012},
  publisher={IEEE}
}

@article{hamilton2017graphsage,
  title={Inductive representation learning on large graphs},
  author={Hamilton, Will and Ying, Zhitao and Leskovec, Jure},
  journal={Advances in neural information processing systems},
  volume={30},
  year={2017}
}

@article{cai2020oversmoothing-9,
  title={A note on over-smoothing for graph neural networks},
  author={Cai, Chen and Wang, Yusu},
  journal={arXiv preprint arXiv:2006.13318},
  year={2020}
}

@article{rusch2023oversmoothing-8,
  title={A survey on oversmoothing in graph neural networks},
  author={Rusch, T Konstantin and Bronstein, Michael M and Mishra, Siddhartha},
  journal={arXiv preprint arXiv:2303.10993},
  year={2023}
}

@article{topping2021oversquashing-1,
  title={Understanding over-squashing and bottlenecks on graphs via curvature},
  author={Topping, Jake and Di Giovanni, Francesco and Chamberlain, Benjamin Paul and Dong, Xiaowen and Bronstein, Michael M},
  journal={arXiv preprint arXiv:2111.14522},
  year={2021}
}

@inproceedings{di2023oversquashing-2,
  title={On over-squashing in message passing neural networks: The impact of width, depth, and topology},
  author={Di Giovanni, Francesco and Giusti, Lorenzo and Barbero, Federico and Luise, Giulia and Lio, Pietro and Bronstein, Michael M},
  booktitle={International Conference on Machine Learning},
  pages={7865--7885},
  year={2023},
  organization={PMLR}
}

@inproceedings{giraldo2023oversquashing-3,
  title={On the trade-off between over-smoothing and over-squashing in deep graph neural networks},
  author={Giraldo, Jhony H and Skianis, Konstantinos and Bouwmans, Thierry and Malliaros, Fragkiskos D},
  booktitle={Proceedings of the 32nd ACM international conference on information and knowledge management},
  pages={566--576},
  year={2023}
}

@misc{tsplib,
  howpublished = {\url{https://www.math.uwaterloo.ca/tsp}},
  year         = {n.d.}
}

@inproceedings{goh2024hierarchical,
  title={Hierarchical neural constructive solver for real-world tsp scenarios},
  author={Goh, Yong Liang and Cao, Zhiguang and Ma, Yining and Dong, Yanfei and Dupty, Mohammed Haroon and Lee, Wee Sun},
  booktitle={Proceedings of the 30th ACM SIGKDD Conference on Knowledge Discovery and Data Mining},
  pages={884--895},
  year={2024}
}

@article{englert2014worst,
  title={Worst case and probabilistic analysis of the 2-Opt algorithm for the TSP},
  author={Englert, Matthias and R{\"o}glin, Heiko and V{\"o}cking, Berthold},
  journal={Algorithmica},
  volume={68},
  number={1},
  pages={190--264},
  year={2014},
  publisher={Springer}
}

@article{babin2007improvements,
  title={Improvements to the Or-opt heuristic for the symmetric travelling salesman problem},
  author={Babin, Gilbert and Deneault, St{\'e}phanie and Laporte, Gilbert},
  journal={Journal of the Operational Research Society},
  volume={58},
  number={3},
  pages={402--407},
  year={2007},
  publisher={Taylor \& Francis}
}

@article{xiao2024nar4tsp,
  title={Reinforcement learning-based nonautoregressive solver for traveling salesman problems},
  author={Xiao, Yubin and Wang, Di and Li, Boyang and Chen, Huanhuan and Pang, Wei and Wu, Xuan and Li, Hao and Xu, Dong and Liang, Yanchun and Zhou, You},
  journal={IEEE Transactions on Neural Networks and Learning Systems},
  year={2024},
  publisher={IEEE}
}

@inproceedings{jin2023pointerformer,
  title={Pointerformer: Deep reinforced multi-pointer transformer for the traveling salesman problem},
  author={Jin, Yan and Ding, Yuandong and Pan, Xuanhao and He, Kun and Zhao, Li and Qin, Tao and Song, Lei and Bian, Jiang},
  booktitle={Proceedings of the AAAI Conference on Artificial Intelligence},
  volume={37},
  number={7},
  pages={8132--8140},
  year={2023}
}

@article{gao2023elg,
  title={Towards generalizable neural solvers for vehicle routing problems via ensemble with transferrable local policy},
  author={Gao, Chengrui and Shang, Haopu and Xue, Ke and Li, Dong and Qian, Chao},
  journal={arXiv preprint arXiv:2308.14104},
  year={2023}
}

@inproceedings{wang2025deitsp,
author = {Wang, Mingzhao and Zhou, You and Cao, Zhiguang and Xiao, Yubin and Wu, Xuan and Pang, Wei and Jiang, Yuan and Yang, Hui and Zhao, Peng and Li, Yuanshu},
title = {An Efficient Diffusion-based Non-Autoregressive Solver for Traveling Salesman Problem},
year = {2025},
booktitle = {Proceedings of the 31st ACM SIGKDD Conference on Knowledge Discovery and Data Mining V.1},
pages = {1469–1480}
}

@article{li2023t2t,
  title={T2t: From distribution learning in training to gradient search in testing for combinatorial optimization},
  author={Li, Yang and Guo, Jinpei and Wang, Runzhong and Yan, Junchi},
  journal={Advances in Neural Information Processing Systems},
  volume={36},
  pages={50020--50040},
  year={2023}
}

@article{zhang2026hybrid,
  title={Hybrid-balance gflownet for solving vehicle routing problems},
  author={Zhang, Ni and Cao, Zhiguang},
  journal={Advances in Neural Information Processing Systems},
  volume={38},
  pages={16797--16821},
  year={2026}
}

@inproceedings{zhang2025adversarial,
  title={Adversarial generative flow network for solving vehicle routing problems},
  author={Zhang, Ni and Yang, Jingfeng and Cao, Zhiguang and Chi, Xu},
  booktitle={International conference on learning representations},
  year={2025}
}

@inproceedings{huang2025rethinking,
  title={Rethinking light decoder-based solvers for vehicle routing problems},
  author={Huang, Ziwei and Zhou, Jianan and Cao, Zhiguang and Xu, Yixin},
  booktitle={International conference on learning representations},
  year={2025}
}

@inproceedings{yi2026radar,
  title={RADAR: Learning to Route with Asymmetry-aware DistAnce Representations},
  author={Yi, Hang and Huang, Ziwei and Ma, Yining and Cao, Zhiguang},
  booktitle={International conference on learning representations},
  year={2026}
}

@article{yang2023diffusion,
  title={Diffusion models: A comprehensive survey of methods and applications},
  author={Yang, Ling and Zhang, Zhilong and Song, Yang and Hong, Shenda and Xu, Runsheng and Zhao, Yue and Zhang, Wentao and Cui, Bin and Yang, Ming-Hsuan},
  journal={ACM computing surveys},
  volume={56},
  number={4},
  pages={1--39},
  year={2023},
  publisher={ACM New York, NY, USA}
}

@inproceedings{cheng2025atom,
  title={Atom: A framework of detecting query-based model extraction attacks for graph neural networks},
  author={Cheng, Zhan and Shen, Bolin and Sha, Tianming and Gao, Yuan and Li, Shibo and Dong, Yushun},
  booktitle={Proceedings of the 31st ACM SIGKDD Conference on Knowledge Discovery and Data Mining V. 2},
  pages={322--333},
  year={2025}
}

@article{wang2025cega,
  title={Cega: A cost-effective approach for graph-based model extraction and acquisition},
  author={Wang, Zebin and Lin, Menghan and Shen, Bolin and Anderson, Ken and Liu, Molei and Cai, Tianxi and Dong, Yushun},
  journal={arXiv preprint arXiv:2506.17709},
  year={2025}
}

@article{li2025intellectual,
  title={Intellectual property in graph-based machine learning as a service: Attacks and defenses},
  author={Li, Lincan and Shen, Bolin and Zhao, Chenxi and Sun, Yuxiang and Zhao, Kaixiang and Pan, Shirui and Dong, Yushun},
  journal={arXiv preprint arXiv:2508.19641},
  year={2025}
}

@article{zhao2026graphip,
  title={GraphIP-Bench: How Hard Is It to Steal a Graph Neural Network, and Can We Stop It?},
  author={Zhao, Kaixiang and Shen, Bolin and Dai, Yuyang and Chakraborty, Shayok and Dong, Yushun},
  journal={arXiv preprint arXiv:2605.12827},
  year={2026}
}

@inproceedings{xu2025autostdiff,
  title={AutoSTDiff: Autoregressive Spatio-Temporal Denoising Diffusion Model for Asynchronous Trajectory Generation},
  author={Xu, Rongchao and Hong, Zhiqing and Wang, Guang},
  booktitle={Proceedings of the 2025 SIAM International Conference on Data Mining (SDM)},
  pages={538--547},
  year={2025},
  organization={SIAM}
}

@article{xu2026synhat,
  title={SynHAT: A Two-stage Coarse-to-Fine Diffusion Framework for Synthesizing Human Activity Traces},
  author={Xu, Rongchao and Jiang, Lin and Yu, Dahai and Li, Ximiao and Wang, Guang},
  journal={arXiv preprint arXiv:2604.14705},
  year={2026}
}

@inproceedings{yu2025uqgnn,
  title={UQGNN: Uncertainty Quantification of Graph Neural Networks for Multivariate Spatiotemporal Prediction},
  author={Yu, Dahai and Zhuang, Dingyi and Jiang, Lin and Xu, Rongchao and Ye, Xinyue and Bu, Yuheng and Wang, Shenhao and Wang, Guang},
  booktitle={Proceedings of the 33rd ACM International Conference on Advances in Geographic Information Systems},
  pages={52--65},
  year={2025}
}

@misc{yu2026health,
  title={HealthMamba: An Uncertainty-aware Spatiotemporal Graph State Space Model for Effective and Reliable Healthcare Facility Visit Prediction},
  author={Dahai Yu and Lin Jiang and Rongchao Xu and Guang Wang},
  journal={arXiv preprint arXiv:2602.05286},
  year={2026},
}


\appendix
\newpage
\section{Proof of Theoretical Results}
\label{appendix:proof}

% \begin{lemma}[Diffusion Distance~\citep{beaini2021directional}]
% Let $q_t(x, y)$ be the transition probability from node $x$ to node $y$ at time $t$ under the continuous-time diffusion process governed by the normalized graph Laplacian $L_{\text{norm}}$. Then,
% \(
% q_t = e^{-t L_{\text{norm}}}
% \)
% is called the \emph{continuous-time heat kernel}, and satisfies the heat equation:
% \(
% \frac{d}{dt} q_t = -L_{\text{norm}} q_t.
% \)
% The \emph{diffusion distance} between two nodes $x$ and $y$ at time $t$ is defined as:
% \[
% d_t(x, y) := \left( \sum_{z \in \mathcal{V}} \left( q_t(x, z) - q_t(y, z) \right)^2 \right)^{1/2}.
% \]
% \end{lemma}

\begin{proof}[Proof of Theorem~\ref{theorem:bi-grad-reduce}]
% We follow a similar approach to the proof of Theorem 2.3 in~\cite{beaini2021directional}. 
We are inspired by the proof strategy of Theorem 2.3 in~\cite{beaini2021directional}.
Let $\phi_1$ denote the first non-trivial eigenvector of $\mL$, where $\mL$ is the normalized Laplacian associated with $\mM$, associated with eigenvalue $\lambda_1$, and suppose $x, y, z \in \mathcal{V}$ satisfy $\phi_1(x) < \phi_1(z)$ and $\phi_1(y) < \phi_1(z)$. Let $(x’, y’)$ be the nodes obtained from $(x, y)$ by simultaneously taking a forward gradient step from $x$ to $z$ under $\mM_\gS$, and a backward gradient step from $y$ to $z$ under $\mM_\gD$. Define diffusion distance under $\mM$ as:
\[
d_k^\mM(x, y)^2 = \sum_{i \geq 1} e^{-2k \lambda_i} \left( \phi_i(x) - \phi_i(y) \right)^2,
\]
where ${ \phi_i }_{i=0}^{n-1}$ are the orthonormal eigenvectors of $\mL$. We aim to show:
\[
d_k^\mM(x’, y’)^2 < d_k^\mM(x, y)^2.
\]
The reduction in distance is primarily due to the contraction along $\phi_1$, since:
\(
(\phi_1(x’) - \phi_1(y’))^2 < (\phi_1(x) - \phi_1(y))^2,
\)
as both $x$ and $y$ move toward the higher-potential node $z$, effectively reducing their separation in the direction of $\phi_1$. As in~\cite{beaini2021directional}, we have bound:
\[
\sum_{i \geq 2} e^{-2k \lambda_i} \left[ (\phi_i(x’) - \phi_i(y’))^2 - (\phi_i(x) - \phi_i(y))^2 \right] \leq e^{-2k \lambda_2} \sum_{i \geq 2} \Delta_i,
\]
where each $\Delta_i$ denotes the change in squared difference on eigencomponent $i$. Since this term decays exponentially faster than $e^{-2k\lambda_1}$, there exists a threshold $C > 0$ such that for $k \geq C$, the overall distance decreases:
\[
d_k^\mM(x’, y’)^2 < d_k^\mM(x, y)^2.
\]
Thus, we can conclude that the anisotropic gradient step yields a contraction in diffusion distance.

\end{proof}

\section{Reproducibility}

% Our code is now available and can be accessed here: \url{https://anonymous.4open.science/r/agdn}.
% Our code is now available: \url{https://anonymous.4open.science/r/agdn}.

\subsection{Data Preprocessing}
We directly used the dataset from the UTSP repository, which is entirely consistent with the setup in previous work~\cite{fu2021generalize}. We first performed preprocessing by rescaling the node coordinates for each sample. We introduced a scale ratio $\epsilon$ and denoted the rescaled node coordinates as $\tilde{v}$. For nodes in a TSP instance, the rescaling is expressed as $\tilde{v} = \epsilon \cdot (v - \bar{v})$, where $\bar{v}$ denotes mean value of coordinates in a TSP instance, and $\epsilon = 2.0$ is set for TSP-100 and TSP-200 tasks, and  $\epsilon = 4.0$ is set for TSP-500 and TSP-1000 tasks.

\subsection{Experimental Settings}

\noindent\textbf{Supervised Learning.} Our settings in the supervised mode are entirely consistent with those described in the GatedGCN~\cite{joshi2019efficient}. Our training and inference processes are entirely consistent to eliminate any discrepancies. 

% Additionally, we also employed the exact same parameters during MCTS searching.

\noindent\textbf{Unupervised Learning.} Our settings in the unsupervised mode are fully consistent with those mentioned in the UTSP~\cite{min2024utsp}, including the MCTS parameters used, which are also identical to those described in the paper. In unsupervised learning, a surrogate loss is proposed to optimize the model, which requires the construction of a soft indicator matrix $\mathbb{T}$. As a result, our anisotropic design becomes unnecessary. Therefore, we removed this mechanism and ensured that the model’s output aligns with the requirements of the surrogate loss. In our model, since anisotropic information is no longer present, we only employ the transition matrix $\mT_{rw}$ to propagate neighbor information during graph diffusion. Additionally, due to the absence of information from two separate spaces, we also omitted the node aggregation process and directly used the node embeddings to predict the soft indicator matrix $\mathbb{T}$.

\noindent\textbf{Sparsification.} \label{appendix:sparse}
To reduce additional noise, we applied sparsification to the transition matrix. It is worth noting that the retained connections significantly exceed those in all current approaches. Specifically, for GatedGNN, it retains 20\% of the edges, whereas our approach retains 70\% of the edges.

\subsection{Implementation of AGDN}
% Model, Parameters, Training settings
\noindent\textbf{Architecture.} For our model, we used three layers, and in each layer we applied BatchNorm followed by Dropout to normalize the node embeddings. Next, we used the ReLU activation function, and a residual layer was incorporated to fuse node embeddings. Finally, The anisotropic node embeddings were concatenated and then mapped to the hidden dimension using a linear layer. We then use this fused embedding to predict heatmap.

\noindent\textbf{Parameters.} The design of our model parameters includes a hidden dimension of 128 and drop probability of 0.5. The diffusion step K is set to 5, and the weights of the PageRank-based diffusion matrix are initialized with a teleport probability of 0.1. 

\noindent\textbf{Training.} During training, we used the Adam optimizer to optimize our model. 
% And the learning rate is configured as $1 \times 10^{-4}$ for TSP-100 and TSP-200, and $5 \times 10^{-5}$ for TSP-500 and TSP-1000. 
The weight decay was set to 1.0 for all cases. Our model is trained for 20 epochs, and the final epoch is used for testing. Specifically, we applied StepLR to decay the learning rate, with a decay factor $\gamma = 0.8$ for every 8 epochs. In addition, we applied gradient clipping using \texttt{clip\_grad\_norm(1.0)} to prevent exploding gradients and stabilize training.
% particularly in the presence of large or unstable updates during optimization.

\subsection{Implementation of Baseline}
\noindent\textbf{GatedGCN.} We implemented GatedGCN entirely based on the code released in \cite{joshi2019efficient} repository. 
% All parameters were configured according to the best settings indicated in the original paper. 
We set the sparse ratio to 0.2, meaning each node remains connected to the closest 20\% of nodes. 
The encoder consisted of 3 layers, with a hidden dimension of 128. 
% The hyperparameters for the training process were consistent with ours. 
Then we standardized our model and GatedGCN model to use Monte Carlo Tree Search for searching the predicted paths. 

\noindent\textbf{UTSP.} For UTSP, we strictly follow the original experimental settings described in~\cite{min2024utsp}. 
% Specifically, we adopt a 2-layer encoding network with a hidden dimension of 64 for TSP-200 and TSP-500, and 128 for TSP-1000. 
The model is trained for 300 epochs with a batch size of 32 for TSP-200 and 64 for TSP-500 and TSP-1000. We use the Adam optimizer with a learning rate of $3 \times 10^{-3}$, weight decay of 0.0, and apply a StepLR scheduler with a step size of 20 and a decay factor $\gamma = 0.8$. 
% The regularization coefficient $C_1$ is set to 20 for TSP-200 and 10 for both TSP-500 and TSP-1000.

% \noindent\textbf{Others.} The results of other baselines in Table 1 are obtained directly from the UTSP~\cite{min2024utsp} and DIFUSCO~\cite{sun2023difusco}. For the baselines presented in Table 3, we strictly follow the results reported in the INViT~\cite{fang2024invit}.
\noindent\textbf{Others.} The results of other baselines in Table 1 are obtained directly from the official implementations and results reported in UTSP~\cite{min2024utsp} and DIFUSCO~\cite{sun2023difusco}. For the baselines presented in Table 3, we strictly follow the official results reported in INViT~\cite{fang2024invit}.
% \subsection{Solution Identification}
% \noindent\textbf{Heatmap Generation.}
% During the inference stage, we extracted the value at the second position of the model’s output heatmap logits as our heatmap. Subsequently, for each heatmap, we retained the top $k$ largest values in each row as input for MCTS.

% \noindent\textbf{MCTS Parameters.}
% The parameters for MCTS were fully configured according to the settings in \cite{min2024utsp}. Some key parameters, however, could be fine-tuned in future work. The $\text{Max Depth}$ parameter determines the maximum depth of $k$-opt moves during MCTS, balancing the complexity of moves and computation time. The $\text{Max Candidate Num}$ parameter restricts the size of the candidate set at each node, balancing search speed and solution quality. Furthermore, the $\text{Param H}$ parameter determines the number of simulations per move, enabling more comprehensive exploration at the expense of higher computational cost. Collectively, these parameters govern the trade-off between solution quality and computational efficiency within the algorithm.

\subsection{Computing Resources}
The training and inference of our model were conducted on an Nvidia RTX 6000 Ada GPU. The CPU used was an AMD EPYC 7763 64-Core @ 2.45 GHz, providing 128 threads. The server is configured with 1000 GB of DDR4 RAM. Monte Carlo Tree Search is performed on the CPU, utilizing 32 threads for evaluation.

% \noindent \textbf{GenAI Usage Disclosure.}
% We confirm that no GenAI tools were used in the creation of the code or data associated with this research. During the writing process, generative AI tools (such as ChatGPT) were used solely for grammar checking and language refinement. No content, analysis, or substantive ideas were generated by AI.

% \clearpage

\section{Supplementary Experiments}
\label{appendix:supplementary_exp}
To further demonstrate the effectiveness of our proposed method, we conduct a comprehensive set of supplementary experiments. In the following, we first analyze the robustness of AGDN, including the impact of hop depth selection, sparsification ratio choices, and different decoder strategies. In addition, we introduce a broader range of baseline methods to provide a more extensive and thorough comparison with our approach.

\subsection{Robustness Analysis: Sensitivity to Hop Depth}
\label{appendix:sup_exp:hop_depth}
We first investigate the sensitivity of AGDN to the choice of hop depth, which controls the range of neighborhood aggregation. We evaluate AGDN on TSP1000 with different hop depths while keeping all other settings fixed.

\begin{table}[h]
\centering
\caption{AGDN performance on TSP1000 under different hop settings.}
\label{tab:tsp1000_hops}
\begin{tabular}{lcc}
\toprule
Setting & Length & Gap (\%) \\
\midrule
hops = 15 & 23.37 & 1.049 \\
hops = 24 & 23.37 & 1.038 \\
hops = 30 & 23.37 & 1.024 \\
\bottomrule
\end{tabular}
\end{table}

As shown in the Table~\ref{tab:tsp1000_hops}, varying the hop depth from 15 to 30 results in negligible differences in both tour length and optimality gap. In particular, the tour length remains identical across all settings, and the gap varies by less than 0.03\%. These results indicate that AGDN is highly robust to the hop depth parameter, and its performance does not rely on a carefully tuned range of neighborhood hop depths.

\subsection{Robustness Analysis: Sensitivity to Sparsification Ratio}
\label{appendix:sup_exp:sp_ratio}
We further analyze the impact of the sparsification ratio ($sp$), which determines the proportion of retained edges during graph sparsification. We conduct experiments on TSP200, TSP500, and TSP1000, varying the sparsification ratio from 0.1 to 1.0. The results are reported in Table~\ref{tab:sp_gap}.

\begin{table}[tb]
\centering
\caption{Performance of AGDN under different $sp$ settings.}
\label{tab:sp_gap}
\begin{tabular}{cccc}
\toprule
$sp$ & TSP200 Gap (\%) & TSP500 Gap (\%) & TSP1000 Gap (\%) \\
\midrule
0.1 & 0.054 & 0.654 & 0.956 \\
0.3 & 0.045 & 0.631 & 1.021 \\
0.5 & 0.063 & 0.642 & 0.988 \\
0.7 & 0.060 & 0.674 & 1.049 \\
1.0 & 0.056 & 0.662 & 0.966 \\
\bottomrule
\end{tabular}
\end{table}

Across all problem scales, AGDN exhibits consistently stable performance under different sparsification ratios. The optimality gaps remain within a narrow range, and no clear degradation is observed even under aggressive sparsification. This suggests that AGDN does not depend on a specific sparsification level and can effectively leverage both sparse and dense graph structures.

\subsection{Robustness Analysis: Different Decoders}
\label{appendix:sup_exp:diff_decoder}
To examine whether the performance gains of AGDN depend on a specific decoding strategy, we further compare AGDN with a widely used graph-based baseline, GatedGCN, under different decoders. Specifically, we consider two commonly adopted local search decoders: 2-opt~\cite{englert2014worst} and Or-opt~\cite{babin2007improvements}, and evaluate all methods on TSP200, TSP500, and TSP1000.

\begin{table}[th]
\centering
\small
\setlength{\tabcolsep}{3pt}
\caption{Comparison of AGDN performance under different decoder strategies, along with comparisons to GatedGCN. The best results are highlighted in \emph{bold}, and the second-best results are \underline{underlined}.}
\label{tab:decoder_gap}
\begin{tabular}{lccc}
\toprule
Model (Decoder) & TSP200 Gap (\%) & TSP500 Gap (\%) & TSP1000 Gap (\%) \\
\midrule
GatedGCN (2-opt)  & $1.530 \pm 0.14$ & $3.553 \pm 0.20$ & $13.91 \pm 1.01$ \\
GatedGCN (Or-opt) & $1.217 \pm 0.11$ & $2.498 \pm 0.17$ & $11.50 \pm 0.87$ \\
AGDN (2-opt)     & $\underline{0.703 \pm 0.48}$ & $\underline{1.274 \pm 0.19}$ & $\underline{1.774 \pm 0.19}$ \\
AGDN (Or-opt)    & $\mathbf{0.445 \pm 0.04}$ & $\mathbf{1.117 \pm 0.06}$ & $\mathbf{1.491 \pm 0.01}$ \\
\bottomrule
\end{tabular}
\end{table}

\vspace{-0.5em}

As shown in the Table~\ref{tab:decoder_gap}, AGDN consistently outperforms GatedGCN across all problem scales and both decoding schemes. Notably, even when using the same decoder, AGDN achieves substantially lower optimality gaps, indicating that its advantage stems from the proposed graph representation and learning mechanism rather than a particular decoding heuristic.
Furthermore, both models benefit from the stronger Or-opt decoder compared to 2-opt, yet the relative performance gap between AGDN and GatedGCN remains large and stable. This observation confirms that the effectiveness of AGDN is decoder-agnostic and does not rely on a specific post-processing or local search strategy.

\subsection{Comparison with Broader Baselines}
\label{appendix:sup_exp:broader_baselines}

\begin{table}[th]
\centering
\small
\setlength{\tabcolsep}{3pt}
\caption{Performance comparison between AGDN and additional baseline methods, where AGDN (SL) denotes the supervised learning setting and AGDN (UL) denotes the unsupervised learning setting. OOM indicates out-of-memory errors. The best results are highlighted in \emph{bold}, and the second-best results are \underline{underlined}.}
\label{tab:broader_baselines}
\begin{tabular}{lccc}
\toprule
Model & TSP200 Gap (\%) & TSP500 Gap (\%) & TSP1000 Gap (\%) \\
\midrule

% AR
PointerFormer~\cite{jin2023pointerformer} & $1.45$ & $12.5$ & $20.7$ \\
BQNCO~\cite{drakulic2024bq} & $4.51$ & $4.02$ & $4.19$ \\
ELG~\cite{gao2023elg} & $1.64$ & $7.16$ & $11.4$ \\
HierTSP~\cite{goh2024hierarchical} & $2.23$ & $12.1$ & $22.8$ \\
% NAR
NAR4TSP~\cite{xiao2024nar4tsp} & $3.07$ & $12.0$ & $16.1$ \\
DEITSP~\cite{wang2025deitsp} & $0.40$ & $2.15$ & $3.68$ \\
% Diffusion
T2TCO~\cite{li2023t2t} & $0.55$ & $2.78$ & $OOM$ \\
% Neural
NeuroLKH~\cite{xin2021neurolkh}    & $0.53$ & $0.83$ & $\mathbf{0.90}$ \\
% Our
AGDN (SL)    & $\underline{0.06}$ & $\underline{0.74}$ & ${1.28}$ \\
AGDN (UL)    & $\mathbf{0.05}$ & $\mathbf{0.64}$ & $\underline{0.96}$ \\
\bottomrule
\end{tabular}
\end{table}

In this section, we further compare AGDN with a broader set of baseline methods. The results in Table~\ref{tab:broader_baselines} show that the proposed AGDN consistently demonstrates strong performance, achieving the best results on both TSP200 and TSP500. Notably, NeuroLKH, as a heuristic-based method, can spend significantly longer decoding time to obtain more accurate solutions. For example, on TSP1000, NeuroLKH requires 1183 seconds for decoding, whereas AGDN requires only 0.97 seconds while still achieving highly competitive and impressive performance.

\end{document}